\documentclass[letterpaper]{article}
\usepackage{authblk}
\usepackage{url}
\usepackage[margin=1in,footskip=0.25in]{geometry}
\usepackage[labelfont=bf]{caption}
\usepackage{graphicx}
\usepackage{amsmath,amssymb,amsfonts,amsmath}
\usepackage{graphicx}
\usepackage{subfigure}
\usepackage{caption}
\usepackage{tikz}
\usepackage{algorithmicx}
\usepackage{algpseudocode}
\usepackage{makecell}
\usepackage{float}
\usepackage[dvipdf]{epsfig}
\usepackage{float}
\usepackage{lipsum}
\usepackage{wrapfig}
\usepackage{cite}
\usepackage{color}
\usepackage{easylist}
\usepackage{lineno}
\usepackage{booktabs}
\usepackage{hyperref}
\usepackage{cleveref}
\usepackage{multirow}
\usepackage{hhline}
\usepackage{soul}
\usepackage{blindtext}
\usepackage{threeparttable}
\usepackage[bbgreekl]{mathbbol}

\newtheorem{algm}{Algorithm}
\newtheorem{lem}{Lemma}[section]
\newtheorem{thm}{Theorem}[section]

\newtheorem{rema}{Remark}[section]

\newenvironment{proof}{{\bf Proof: \/}}
{\vspace{-.3in}\begin{flushright}{\vrule height5pt width3pt
depth0pt} \end{flushright}}
{\vspace{-.3in}\begin{flushright}{\vrule height5pt width3pt
depth0pt} \end{flushright}}

\newcommand{\eqn}[1]{
\begin{eqnarray}
#1
\end{eqnarray}
}

\newcommand{\tb}[1]{
{\bf #1}
}

\newcommand{\norm}[1]{
\left\lVert #1\right\rVert
}

\newcommand{\ip}[1]{
\left\langle #1 \right\rangle
}

\newcommand{\bs}[1]{
\boldsymbol #1
}

\title{Efficient kernel surrogates for neural network-based regression}

\author[1]{Saad Qadeer} 
\author[1]{Andrew Engel}
\author[1]{Amanda Howard}
\author[1,2]{Adam Tsou}
\author[1]{Max Vargas}
\author[1,3,4, *]{Panos Stinis}
\author[1,3,5, *]{Tony Chiang} 
\affil[1]{Pacific Northwest National Laboratory}
\affil[2]{Stony Brook University}
\affil[3]{The University of Washington}
\affil[4]{Brown University}
\affil[5]{The University of Texas, El Paso}
\affil[*]{{\tt panagiotis.stinis@pnnl.gov}, {\tt tony.chiang@pnnl.gov}}

\begin{document}

\maketitle

\begin{abstract}
Despite their immense promise in performing a variety of learning tasks, a theoretical understanding of the effectiveness and limitations of Deep Neural Networks (DNNs) has so far eluded practitioners. This is partly due to the inability to determine the closed forms of the learned functions, making it harder to assess their precise dependence on the training data and to study their generalization properties on unseen datasets. Recent work has shown that randomly initialized DNNs in the infinite width limit converge to kernel machines relying on a Neural Tangent Kernel (NTK) with known closed form. These results suggest, and experimental evidence corroborates, that empirical kernel machines can also act as surrogates for finite width DNNs. The high computational cost of assembling the full NTK, however, makes this approach infeasible in practice, motivating the need for low-cost approximations. In the current work, we study the performance of the Conjugate Kernel (CK), an efficient approximation to the NTK that has been observed to yield fairly similar results. For the regression problem of smooth functions and classification using logistic regression, we show that the CK performance is only marginally worse than that of the NTK and, in certain cases, is shown to be superior. In particular, we establish bounds for the relative test losses, verify them with numerical tests, and identify the regularity of the kernel as the key determinant of performance. In addition to providing a theoretical grounding for using CKs instead of NTKs, our framework provides insights into understanding the robustness of the various approximants and suggests a recipe for improving DNN accuracy inexpensively. We present a demonstration of this on the foundation model GPT-2 by comparing its performance on a classification task using a conventional approach and our prescription. We also show how our approach can be used to improve physics-informed operator network training for regression tasks as well as  convolutional neural network training for vision classification tasks.

\end{abstract}

\noindent \textbf{Keywords:}  Finite-width neural networks, Empirical kernel machines, Neural Tangent Kernel, Function regression, Logistic regression, Generalization errors, Foundation models

\section{Introduction}\label{SecIntro}

Deep Neural Networks (DNNs) have shown immense promise in performing a variety of learning tasks including, among others, image classification \cite{krizhevsky2012imagenet,he2016deep}, natural language processing \cite{mikolov2013efficient,kim2014convolutional,vaswani2017attention,kenton2019bert}, function approximation \cite{yarotsky2018optimal,daubechies2022nonlinear}, solving differential equations \cite{karniadakis2021physics,lu2021learning,kovachki2021universal}, etc. However, a theoretical understanding of their effectiveness and limitations has proven elusive, thereby hindering their universal acceptance in practical applications and limiting their usage in scientific computing. This is partly down to their somewhat unwieldy architectures and complex loss landscapes that do not permit precise closed-form characterizations of the optimally learned functions, thus making it harder to assess their precise dependence on the training data and to study their generalization properties on test sets.

Recent work has established that randomly initialized DNNs in the infinite width limit converge to kernel machines relying on a deterministic Neural Tangent Kernel (NTK) \cite{lee2019wide,arora2019exact}. The NTK for a DNN, defined as the Gram matrix of the Jacobian of the DNN with respect to the network parameters, controls the evolution of the DNN training \cite{jacot2018neural}. In the infinite width limit, the NTK can be shown to not deviate from its known closed form. As a consequence, the result of fully training an infinite width DNN can be {\it a priori} identified as a kernel machine employing the known NTK. These insights also allow one to study the performance of the DNN architecture in the extremal over-parameterized limit and benefit from the double-descent phenomenon \cite{liu2020toward,bartlett2021deep,belkin2021fit}.

It should be emphasized that this limiting regime does not leverage the changes in the DNN during training and hence fails to make use of any features learned by the architecture. However, these results suggest the tantalizing possibility that empirical kernel machines could also act as surrogates for finite width DNNs. As shown elsewhere \cite{engel2023robust} as well as later in the article, under certain assumptions this claim is not only corroborated by strong evidence, both mathematical (equation \eqref{NN_NTK} and Sections \ref{SecKLMathAnalysis} and \ref{SecKRMathAnalysis}) and experimental (Section \ref{SecNumTests}), but also considerably improved upon. This enables a precise determination of the superior learned functions and a thorough study of their performance on unseen data.

The high computational cost of assembling the complete NTK for DNNs employed in practical applications \cite{novak2018bayesian,novak2022fast}, however, makes this approach infeasible, motivating the need for low cost approximations. In this article, we study the performance of the Conjugate Kernel (CK) \cite{fan2020spectra}, a ``zeroth-order'' approximation to the NTK, that has been observed to yield fairly similar results (see Section \ref{SecNumTests}). For the problems of regression of smooth functions and logistic regression, we prove that the CK approximations are only marginally worse than the NTK approximations and, in some cases, are much superior. In particular, we establish bounds for the errors, verify them with numerical tests, and identify the regularity of the kernel as the key determinant of performance. In addition to providing a theoretical grounding for using CKs instead of NTKs, our framework provides insights into understanding the robustness of the various approximations and suggests a recipe for boosting DNN and physics-informed Deep Operator Network (DeepONet, \cite{lu2021learning, wang2021learning}) accuracy inexpensively. We duly present a demonstration of this on GPT-2 by comparing its performance on a classification task using a conventional approach and our prescription (see Section \ref{SecFoundModels}).

\section{Preliminaries}\label{SecPrelim}

In this section, we introduce our notation and the define the basic terms and tools. We denote by $f_\text{NN}: \mathbb{R}^{d_0} \to \mathbb{R}^{d_{L+1}}$ a fully connected network (FCN) with $L$ hidden layers of the form
\eqn{
\bs{z}^{(l)} &=& \sigma \left(W^{(l)}\bs{z}^{(l-1)} + \tb{b}^{(l)} \right), \quad 1 \leq l \leq L, \nonumber\\
\bs{z}^{(L+1)} &=& W^{(L+1)}\bs{z}^{(L)} + \tb{b}^{(L+1)}. \label{NN}
}

Here, $\sigma$ is the activation function, $W^{(l)} \in \mathbb{R}^{d_l \times d_{l-1}}$ and $\tb{b}^{(l)} \in \mathbb{R}^{d_l}$ are the trainable parameters, $\bs{z}^{(0)} \in Z$ is the input, and $\bs{z}^{(L+1)} \in \mathbb{R}^{d_{L+1}}$ is the output, i.e., $f_\text{NN}\left(\bs{z}^{(0)}\right) = \bs{z}^{(L+1)}$. We denote all the trainable parameters $\{W^{(l)},\tb{b}^{(l)} \}_{1 \leq l \leq L+1}$ by $\theta$ and their number by $|\theta|$ for brevity, and indicate the dependence of the neural network on them by writing $f_\text{NN}^{\theta}$. Since this is the only type of neural network (NN) we shall consider in this article, henceforth we refer to FCNs of this type as NNs.

Architectures of this type can be deployed for various tasks. For instance, let $f:Z \to \mathbb{R}$ be a given function, where $Z$ is a compact subset of $\mathbb{R}^{d_0}$, with training input-output pairs $\{(\bs {x}_i,y_i)\}_{0 \leq i \leq N} \subset Z \times \mathbb{R}$ (so that $y_i = f(\bs{x}_i)$). In order to solve the function regression problem, the parameters can be trained so as to minimize the weighted squared loss
\eqn{
L_{0,\text{NN}}\left(\theta\right) = \sum_{i = 0}^N w_i\left| y_i - f^{\theta}_\text{NN}\left(\bs{x}_i\right) \right|^2, \label{RSqLoss}
}
for some non-negative weights $\{w_i\}_{0 \leq i \leq N}$, and yield an approximation $f_\text{NN}^{\theta^*}$ to the target function. We further define the norms
\eqn{
\norm{g}_0 := \left(\sum_{i = 0}^N w_i\left|g(\bs{x}_i)\right|^2 \right)^{1/2}, \qquad \norm{g}_1 := \left(\sum_{i = 0}^M v_i\left|g(\bs{y}_i)\right|^2 \right)^{1/2} \label{NormDefn}
}
for any $g :Z \to \mathbb{R}^{d_{L+1}}$, where $\{(\bs{y}_i,v_i)\}_{0 \leq i \leq M}$ is a collection of test nodes and associated non-negative weights. The training and test errors are then $\norm{f - f^{\theta^*}_\text{NN}}_0 = \left[L(\theta^{*}) \right]^{1/2}$ and $\norm{f - f_\text{NN}^{\theta^*}}_1$ respectively.

Similarly, we can employ NNs for classifications tasks. Let $\{\left({\bs x}_{i},\chi_i\right)\}_{0 \leq i \leq N} \subset Z \times \{0,1\}$ be the training set, where $\chi_i$ is the class label assigned to $\bs{x}_i$. Set $d_{L+1} = 2$ (so that $f^\theta_\text{NN}:\mathbb{R}^{d_0} \to \mathbb{R}^2$) and apply the softmax function to get
\eqn{
p^{\theta}_{\text{NN}}(\bs{z}) = \frac{\exp\left(f^{\theta}_{\text{NN},1}(\bs{z})\right)}{\sum_{i = 1}^2 \exp\left(f^{\theta}_{\text{NN},i}(\bs{z})\right)} = \frac{1}{1 + \exp\left[-\left(f^{\theta}_{\text{NN},1}(\bs{z}) - f^{\theta}_{\text{NN},2}(\bs{z}) \right)\right]}. \label{LPredNN}
}

This is the likelihood, according to this model, of $\bs{z} \in Z$ belonging to the class labelled $1$. The training requires the minimization of the logistic loss (i.e., the cross-entropy)
\eqn{
{\ell}_{0,\text{NN}}(\theta) = -\sum_{i = 0}^{N} \chi_{i}\ln\left(p^{\theta}_\text{NN}(\bs{x}_{i})\right) + (1 - \chi_{i})\ln\left(1 - p^{\theta}_\text{NN}(\bs{x}_{i})\right), \label{LNNTrLoss}
}
with the corresponding test loss
\eqn{
{\ell}_{1,\text{NN}}(\theta) = -\sum_{i = 0}^{M} \mu_{i}\ln\left(p^{\theta}_\text{NN}(\bs{y}_{i})\right) + (1 - \mu_{i})\ln\left(1 - p^{\theta}_\text{NN}(\bs{y}_{i})\right), \label{LNNTeLoss}
}
where $\{\mu_i\}_{0 \leq i \leq M} \subset \{0,1\}$ are the class labels of the test nodes.

These tasks can also be performed with the aid of kernel methods \cite{boser1992training,meanti2020kernel}. Given a valid kernel function $\text{ker}: Z \times Z \to \mathbb{R}$, define the kernel matrix $H_\text{ker} = \begin{pmatrix} \text{ker}(\bs{x}_i , \bs{x}_j)
\end{pmatrix}_{0 \leq i,j \leq N}$, and the weighted kernel matrix $\widehat{H}_\text{ker} = W^{1/2}H_\text{ker}W^{1/2}$, where $W^{1/2} = \text{diag}(\sqrt{w_0},\hdots,\sqrt{w_N})$. The corresponding kernel approximation for the function regression problem above seeks $\bs{\delta}^{*,\text{ker}} \in \mathbb{R}^{N+1}$ to minimize
\eqn{
L_{0,\text{ker}}(\bs{\delta}) = \sum_{i = 0}^N w_i\left| y_i - \sum_{j = 0}^N\bs{\delta}_j \text{ker}(\bs{x}_j,\bs{x}_i) \right|^2. \label{KRTrLoss}
} 

It possesses the closed form solution
\eqn{
f_\text{ker}(\bs{z}) = \bs{y}^\top W^{1/2} \left(\widehat{H}_\text{ker}\right)^\dagger W^{1/2} \begin{pmatrix}
\text{ker}(\bs{x}_j, \bs{z})
\end{pmatrix}_{0 \leq j \leq N}, \label{KRApprox}
} 
where $\bs{y} = \begin{pmatrix} y_i \end{pmatrix}_{0 \leq i \leq N}$ contains the training target values, with training and test losses $\norm{f - f_\text{ker}}_0$ and $\norm{f - f_\text{ker}}_1$ respectively.

The logistic regression problem can likewise be solved by finding an $\bs{\alpha}^{*,\text{ker}} \in \mathbb{R}^{N+1}$ that minimizes
\eqn{
\ell_{0,\text{ker}}(\bs{\alpha}) = \sum_{i = 0}^{N}  \chi_{i} \ln\left[1 + \exp(-H_{\text{ker},(i,:)}\bs{\alpha})\right] + (1-\chi_{i})\ln\left[1 + \exp(H_{\text{ker},(i,:)}\bs{\alpha})\right] , \label{KLTrLoss} 
} 
thus yielding the likelihood
\eqn{
p_\text{ker}(\bs{z}) = \left[1 + \exp\left(-\sum_{j = 0}^{N} \text{ker}(\bs{z},\bs{x}_{j}) \bs{\alpha}^{*,\text{ker}}_{j}\right) \right]^{-1}, \label{KLPred}
}
according to this classifier, of $\bs{z} \in Z$ lying in class 1. The test loss $\ell_{1,\text{ker}}$ is defined similarly to \eqref{LNNTeLoss} with \eqref{LPredNN} replaced by \eqref{KLPred}.

In case $\text{ker}$ is induced by a feature map $\Phi_\text{ker}:Z \to \mathbb{R}^{K}$, so that $\text{ker}(\bs{x},\bs{z}) = \Phi_\text{ker}(\bs{x})^\top\Phi_\text{ker}(\bs{z})$, \eqref{KRApprox} can be understood as an orthogonal projection on $\mathcal{S}_\text{ker} = \text{span}\left(\left\{\Phi_{\text{ker},k}\right\}_{k = 1}^K \right)$ with respect to $\ip{\cdot,\cdot}_0$, the inner product that induces $\norm{\cdot}_0$; a technique for obtaining an optimal representation of this projection is presented in Appendix \ref{AppBFs}. Equivalently, we can interpret \eqref{KRApprox} as weighted linear regression on the data-points $\left\{ \left(\Phi_\text{ker}(\bs{x}_i) , y_i\right) \right\}_{0 \leq i \leq N}$, i.e., it implicitly finds $\bs{\gamma}^{*,\text{ker}} \in \mathbb{R}^K$ that minimizes the training loss
\eqn{
\widetilde{L}_{0,\text{ker}}(\bs{\gamma}) = \sum_{i = 0}^N w_i\left| y_i - \bs{\gamma}^\top \Phi_\text{ker}(\bs{x}_i) \right|^2, \label{KREqTrLoss}
} 
so the approximation \eqref{KRApprox} can also be written as 
\eqn{
f_\text{ker}(\bs{z}) = \left(\bs{\gamma}^{*,\text{ker}}\right)^\top\Phi_\text{ker}(\bs{z}). \label{KREqApprox}
}

This equivalence between the apparently more general form $\bs{\gamma}^\top\Phi_\text{ker}(\bs{z})$ in \eqref{KREqApprox} and the kernel form $\bs{\delta}^\top \begin{pmatrix} \text{ker}(\bs{x}_j , \bs{z}) \end{pmatrix}_{0 \leq j \leq N}$ in \eqref{KRApprox} can be viewed as a consequence of the representer theorem \cite{scholkopf2001generalized} or, equivalently, as a corollary of the (easily proven) fact that 
\eqn{
\text{range}(A^\top A) = \text{range}(A^\top), \label{RangeRes}
}
for any matrix $A$: applied to the weighted feature map matrix $\widehat{\Xi}_\text{ker} = \left(\Phi_{\text{ker},m}(\bs{x}_{i})w_i^{1/2}\right)_{1 \leq m \leq K, 0 \leq i \leq N} \in \mathbb{R}^{K \times (N+1)}$, we deduce that for an optimal $\bs{\gamma}^{*,\text{ker}}$, there exists some $\bs{\delta}^{*,\text{ker}} \in \mathbb{R}^{N+1}$ such that
\eqn{
\widehat{\Xi}_\text{ker}^\top\bs{\gamma}^{*,\text{ker}} = \widehat{\Xi}_\text{ker}^\top \widehat{\Xi}_\text{ker}\bs{\delta}^{*,\text{ker}} = \widehat{H}_\text{ker}\bs{\delta}^{*,\text{ker}}, \label{GammaDeltaRel}
}
since $\widehat{H}_\text{ker} = \widehat{\Xi}_\text{ker}^\top \widehat{\Xi}_\text{ker}$, and hence that we can restrict our search for $\bs{\gamma}^{*,\text{ker}}$ in $\text{range}\left(\widehat{\Xi}_\text{ker}\right)$. This implies that
\eqn{
\norm{f - f_\text{ker}}_\nu^2 = \widetilde{L}_{\nu,\text{ker}}(\bs{\gamma^{*,\text{ker}}}), \label{KREqLosses}
}
for $\nu = 0,1$, where the test loss $\widetilde{L}_{1,\text{ker}}$ is defined in the same manner as \eqref{KREqTrLoss} with the test points and weights replacing the training ones.

Similarly, solving the kernel logistic problem can be interpreted as a linear classifier in the feature space, i.e., identifying the decision boundary 
\eqn{
\left\{\bs{x} \in \mathbb{R}^{d_0}: \left(\bs{\beta^{*,\text{ker}}}\right)^\top \Phi_\text{ker}(\bs{x}) = 0\right\}, \label{KLBndry}
}
which can be seen as a hyperplane in $\mathbb{R}^K$, by finding $\bs{\beta}^{*,\text{ker}} \in \mathbb{R}^K$ that minimizes
\eqn{
\widetilde{\ell}_{0,\text{ker}}(\bs{\beta}) = \sum_{i = 0}^{N} \chi_{i} \ln\left[1 + \exp\left(-\bs{\beta}^\top \Phi_\text{ker}(\bs{x}_i) \right)\right] + (1-\chi_{i})\ln\left[1 + \exp\left(\bs{\beta}^\top \Phi_\text{ker}(\bs{x}_i) \right)\right]. \label{KLEqTrLoss} 
} 

Following an argument similar to that used above, we have
\eqn{
\ell_{\nu,\text{ker}}(\bs{\alpha}^{*,\text{ker}}) = \widetilde{\ell}_{\nu,\text{ker}}(\bs{\beta}^{*,\text{ker}}), \label{KLEquivLosses}
}
for $\nu = 0,1$.


\section{Problem formulation}\label{SecProb}
Given a (fully or partially trained or untrained) neural network $f_\text{NN}:\mathbb{R}^{d_0} \to \mathbb{R}$, the corresponding Neural Tangent Kernel (NTK) is the kernel induced by the feature map
\eqn{
\Phi_{\text{NTK},k} := \frac{\partial f^{\theta}_\text{NN}}{\partial \theta_k}, \quad 1 \leq k \leq |\theta|, \label{NTKFeaMap}
}
i.e., the Jacobian of the neural network with respect to the trainable parameters \cite{jacot2018neural}. In the infinite width limit, it has been established that, over the course of training a randomly initialized NN, the NTK possesses an unchanging closed form, and that the NN converges to the corresponding NTK machine \cite{lee2019wide}. This highlights the linear dependence of the NN on the underlying parameters, to wit,
\eqn{
f^{\theta}_\text{NN} = f^{\theta_0}_\text{NN} + \left(\theta - \theta_0\right)^\top\left.\left(\frac{\partial f^{\theta}_\text{NN}}{\partial \theta_k}\right)\right|_{\theta = \theta_0}. \label{NNTaylor}
} 

In fact, since $f^{\theta_0}_\text{NN}$ is a linear combination of the neurons in the last hidden layer (from \eqref{NN}) and which correspond the Jacobian entries corresponding to the last layer of parameters, it can be subsumed in the second term in \eqref{NNTaylor} to yield
\eqn{
f^{\theta}_\text{NN}  = \bs{\gamma}^\top \Phi_\text{NTK} \label{NN_NTK}
}
for some $\bs{\gamma} \in \mathbb{R}^{|\theta|}$. Since this is precisely the form of an NTK approximation (from the equivalence of \eqref{KRApprox} and \eqref{KREqApprox}), we deduce that, in the infinite width limit, a fully trained NN is identical to an optimal regressor employing the NTK. 

However, using this insight to obtain performance guarantees and assess limitations of finite width network runs into two problems. First, the assumption of linear dependence on parameters in \eqref{NNTaylor} does not hold \cite{liu2020linearity}, so the optimal NTK machine can only be viewed as the best approximator in the tangent space of the NN at $\theta_0$, and not on the manifold of all NNs. Secondly, computing the NTK in practice is a computationally daunting task \cite{novak2018bayesian,novak2022fast}. In the infinite width regime and under random initialization, a closed form for the NTK can be calculated. However, empirical kernels do not possess such characterizations in general, necessitating the use of low-cost approximations.

The Conjugate Kernel (CK) is defined as the kernel induced by the Jacobian only with respect to the parameters $\{W^{(L+1)},\tb{b}^{(L+1)}\}$ in the last layer \cite{fan2020spectra}. Consequently, its feature map can be seen as a truncation of the full Jacobian employed by the NTK. Observe that we have set $d_{L+1} = 1$ so $W^{(L+1)} \in \mathbb{R}^{d_L}$ and $\tb{b}^{L+1} \in \mathbb{R}$, and hence
\eqn{
\Phi_{\text{CK},k} := \left\{\begin{matrix}
\frac{\partial f^{\theta}_\text{NN}}{\partial W^{(L+1)}_k}, & \text{if } 1\leq k \leq d_L, \\ \\ 
\frac{\partial f^{\theta}_\text{NN}}{\partial \tb{b}^{(L+1)}}, & \text{if } k = d_L+1. \\
\end{matrix} \right. \label{CKFeaMap}
}

A moment's thought then reveals that
\eqn{
\text{CK}(\bs{z},\widetilde{\bs{z}}) := \Phi_{\text{CK}}(\bs{z})^\top \Phi_{\text{CK}}(\widetilde{\bs{z}}) = 1 + \bs{z}^{(L)} \cdot \widetilde{\bs{z}}^{(L)}. \label{CKDefn}
}

\begin{rema}\label{RemaKCla}
For an NN trained for the binary classification problem described in the previous section, we have $d_{L+1} = 2$. In this setting, we define the NTK and CK for $f_\text{diff} := f^{\theta}_{\text{NN},1} - f^{\theta}_{\text{NN},2}$ since, as shown in \eqref{LPredNN}, this difference dictates the output of the logistic predictor. 
\end{rema}

Denoting by $\text{E}$ the contribution towards the NTK from every layer but the last allows us to write
\eqn{
\text{NTK}(\bs{z},\widetilde{\bs{z}}) = \text{CK}(\bs{z},\widetilde{\bs{z}})  + \text{E}(\bs{z},\widetilde{\bs{z}}). \label{NTK_CK_exp}
}

The CK can therefore be viewed as a ``zeroth-order'' approximation to the NTK. All these components are valid kernels that can be calculated from a given NN. However, the cost of computing the CK at a pair of points is the same as that of evaluating the NN at those points since, according to \eqref{CKDefn}, it makes use of the neuron values at the last hidden layer. Moreover, numerical evidence suggests that, while the NTK expectedly outperforms the CK on test sets, the improvement is only marginal and is dwarfed by the superiority both kernels exhibit over the NN architecture from which they are derived (see Section \ref{SecNumTests}). In other words, the CK appears to yield substantive gains in accuracy over a given finite width neural network while being easier to calculate than the NTK and much cheaper to train further than the NN. This raises the intriguing possibility that the CK can serve as a low-cost proxy for the NTK, i.e., lead to a comparable boost in performance over the NN, while being significantly less expensive to assemble than the NTK and considerably easier to train than the NN. Since training the CK is the same as further training the last parameter layer of the NN, our work provides a recipe for boosting NN accuracy and also sheds light on the robustness (or lack thereof) conferred on the approximations by the choice of the activation function.

In the following sections, we provide a mathematical treatment of this phenomenon for two types of problems: regression of smooth functions and classification using logistic regression. For ease of exposition and analysis, we concentrate on these problems in low dimensions with structured training and test nodes. The analyses of these problems are presented in Sections \ref{SecKRMathAnalysis} and \ref{SecKLMathAnalysis}, and are complemented by numerical experiments presented in Section \ref{SecNumTests}.

\section{Function regression}\label{SecKRMathAnalysis}

\subsection{Preliminaries}\label{SubSecKRPrelim}

We consider the approximation problem for a smooth function $f: [a,b] \to \mathbb{R}$ whose values are known at the equispaced training nodes $x_j = a + j\Delta x$ for $0 \leq j \leq N$, where $\Delta x = (b-a)/N$. The performance is assessed on the equispaced test nodes $y_j = a + j\Delta y$ for $0 \leq j \leq M$, where $\Delta y = (b-a)/M$; we take $M = \tau N$, for some integer $\tau > 1$, and set 
\eqn{
w_i = \left\{ \begin{matrix}
\Delta x, & 1 \leq i \leq N-1 \\
\Delta x/2, & i = 0, N \\
\end{matrix} \right., \qquad v_i = \left\{ \begin{matrix}
\Delta y, & 1 \leq i \leq M-1 \\
\Delta y/2, & i = 0, M \\
\end{matrix} \right., \label{KRWts}
}
so for any $g: [a,b] \to \mathbb{R}$, we have
\eqn{
\norm{g}_0 = \left(\frac{\Delta x}{2} \sum_{j = 0}^{N-1} \left(g(x_j)^2 + g(x_{j+1})^2\right)\right)^{1/2}, \quad	
\norm{g}_1 = \left(\frac{\Delta y}{2} \sum_{j = 0}^{M-1} \left(g(y_j)^2 + g(y_{j+1})^2\right)\right)^{1/2}. \label{NormReDefn}	
}

These can be recognized as the trapezoidal rule approximations to the $L^2$ norm on $[a,b]$; recall that this approximation is second-order accurate with respect to the step-size in general and spectrally accurate for periodic functions. Moreover, these choices serve to considerably simplify the analysis: for instance, since $y_{\tau j} = x_j$ for $0 \leq j \leq N$, we have, for any function $g$,
\eqn{
\norm{g}_0^2 = \frac{\Delta x}{2} \sum_{j = 0}^{N-1} \left(g(x_j)^2 + g(x_{j+1})^2\right) &\leq& \frac{\Delta x}{2}\sum_{j = 0}^{N-1} \sum_{i = 0}^{\tau} g(y_{\tau j + i})^2   = \left(\frac{M}{N}\right) \norm{g}_1^2 \nonumber
}
so that
\eqn{
\norm{g}_0 &\leq& \tau^{1/2}\norm{g}_1. \label{NormIneq1}
}

This straightforward result allows us to describe the size of the function on the lower-resolution training grid in terms of the higher-resolution test grid. The next two results enable us to move in the opposite direction, i.e., bounding the norm on the finer grid by the norm on the coarser grid. This naturally requires some assumptions about the behaviour of the function away from the training nodes; the following lemmas consider two different settings.  

\subsection{Two Lemmas}\label{SubSecLemmas}

\begin{lem}\label{LemMono}
Suppose that $g$ is monotonic on every sub-interval $[x_i,x_{i+1}]$. Then,
\eqn{
\norm{g}_1 \leq \sqrt{2}\norm{g}_0. \label{LemMonoRes}
}
\end{lem}

\begin{proof}\label{ProofMono}
The monotonicity of $g$ on every sub-interval implies 
\eqn{
|g(y_{\tau i + k})| \leq \text{max}\left\{ |g(x_i)| , |g(x_{i+1})| \right\}, \quad 1 \leq k \leq \tau-1, \ 0 \leq i \leq N-1, \label{MonoImpl1}
}
with the result
\eqn{
\norm{g}_1^2 &=& \frac{\Delta y}{2} \sum_{j = 0}^{M-1} g(y_j)^2 + g(y_{j+1})^2 \nonumber\\
&=& \frac{\Delta y}{2} \sum_{j = 0}^{N-1} \left[g(x_j)^2 + g(x_{j+1})^2\right] + \frac{\Delta y}{2}\sum_{i = 0}^{N-1} \sum_{k = 1}^{\tau-1} g(y_{\tau i + k})^2 \nonumber\\
&\leq&  \frac{\Delta y}{2} \sum_{j = 0}^{N-1} \left[g(x_j)^2 + g(x_{j+1})^2\right] + (\tau-1)\frac{\Delta y}{2}\sum_{i = 0}^{N-1}  \text{max}\left\{ g(x_i)^2 , g(x_{i+1})^2 \right\} \nonumber\\
&\leq& \frac{1}{\tau} \norm{g}_0^2 + \frac{2(\tau-1)}{\tau}\norm{g}_0^2 \nonumber\\
&\leq& 2 \norm{g}_0^2, \nonumber
} 
and hence $\norm{g}_1 \leq \sqrt{2}\norm{g}_0$.	
\end{proof}

Combining this with \eqref{NormIneq1}, we have the norm equivalence
\eqn{
\tau^{-1/2}\norm{g}_0 \leq \norm{g}_1 \leq \sqrt{2}\norm{g}_0. \label{MonoImpl2}
}

\begin{lem}\label{LemLip}
Suppose that $g$ is Lipschitz on $[a,b]$ with Lipschitz constant $L_g$. Then,
\eqn{
\norm{g}^2_1 \leq 4\left(\norm{g}_0^2 + \frac{2(b-a)^2L_g^2}{N^2}\right). \label{LemLipRes}
}
\end{lem}

\begin{proof}\label{ProofLip}
The Lipschitz property implies that
\eqn{
|g(y_{\tau j + k})| \leq \text{max}\left\{ |g(x_j)| + kL_g\Delta y  , |g(x_{j+1})| + (\tau-k)L_g\Delta y\right\} \nonumber
}
for $0 \leq k \leq \tau$. It follows that
\eqn{
g(y_{\tau j + k})^2 &\leq& 2\left[ g(x_j)^2 + k^2L_g^2(\Delta y)^2 + g(x_{j+1})^2 + (\tau-k)^2L_g^2(\Delta y)^2 \right] \nonumber\\
&\leq& 2\left[ g(x_j)^2 + g(x_{j+1})^2  + 2L_g^2(\Delta x)^2  \right], \nonumber
}
so for $0 \leq l \leq \tau-1$, we have
\eqn{
g(y_{\tau j + l})^2 + g(y_{\tau j + l +1})^2 \leq 4\left[ g(x_j)^2 + g(x_{j+1})^2  + 2L_g^2(\Delta x)^2  \right]. \nonumber
}

As a result,
\eqn{
\sum_{l = 0}^{\tau - 1} g(y_{\tau j + l})^2 + g(y_{\tau j + l +1})^2 &\leq& 4\tau\left[ g(x_j)^2 + g(x_{j+1})^2  + 2L_g^2(\Delta x)^2  \right] \nonumber\\
\Rightarrow \sum_{j = 0}^{N-1}\sum_{l = 0}^{\tau - 1} g(y_{\tau j + l})^2 + g(y_{\tau j + l +1})^2 &\leq& 4\tau \sum_{j = 0}^{N-1}\left[ g(x_j)^2 + g(x_{j+1})^2  + 2L_g^2(\Delta x)^2  \right], \nonumber
}
and hence
\eqn{
M\norm{g}_1^2 \leq 4N\tau \left[ \norm{g}_0^2  + 2L_g^2(\Delta x)^2  \right] \Rightarrow \norm{g}^2_1 \leq 4\left(\norm{g}_0^2 + \frac{2(b-a)^2L_g^2}{N^2}\right). \nonumber
}

\end{proof}


\subsection{Approximation properties}
For a smooth target function $f: [a,b] \to \mathbb{R}$, let $f_\text{CK}$ and $f_\text{NTK}$ be the CK and NTK approximations to $f$ with respect to the training nodes. Recall that \eqref{KRApprox} can be viewed as an orthogonal projection $\mathcal{P}_\text{ker}$ on $\mathcal{S}_\text{ker}$, the span of the feature map components corresponding to the kernel. Since $\mathcal{S}_\text{CK} \subset \mathcal{S}_\text{NTK}$, we have
\eqn{
\norm{f - f_\text{NTK}}_0 \leq \norm{f - f_\text{CK}}_0. \label{TrBound}
}

Next, we establish that the error of the CK approximation may also be controlled by that of the NTK approximation for most target functions. More precisely, we define, for any $\beta \leq 1$
\eqn{
\mathcal{F}_{\beta} = \left\{f: [a,b] \to \mathbb{R} \text{ is smooth with } \norm{\mathcal{P}_\text{NTK}f}_0 \leq \beta\norm{f}_0 \right\}. \label{FbetaDef}
}

We note that since $\mathcal{P}_\text{NTK}$ is an orthogonal projection, its norm is one and so $\mathcal{F}_1$ contains all smooth real-valued functions. However, most functions of interest would not lie in $\mathcal{S}_\text{NTK}$ and hence would belong to some $\mathcal{F}_\beta$ with $\beta < 1$.  

\begin{lem}\label{LemProjBd}
Let $f \in \mathcal{F}_{\beta}$. Then, 
\eqn{
\norm{\mathcal{P}_\text{NTK}(f - f_\text{CK})}_0 \leq \beta\norm{f - f_\text{CK}}_0. \label{FbetaImpl}
}
\end{lem}

\begin{proof}
We have
\eqn{
\norm{f}_0^2 = \norm{\mathcal{P}_\text{NTK}f}_0^2 + \norm{(I - \mathcal{P}_\text{NTK})f}_0^2 = \norm{f_\text{NTK}}_0^2 + \norm{f - f_\text{NTK}}_0^2 \label{fsplit}
}
and
\eqn{
\norm{f - f_\text{CK}}_0^2 &=& \norm{\mathcal{P}_\text{NTK}(f - f_\text{CK})}_0^2 + \norm{(I - \mathcal{P}_\text{NTK})(f - f_\text{CK})}_0^2 \nonumber\\
&=& \norm{f_\text{NTK} - f_\text{CK}}_0^2 + \norm{f - f_\text{NTK}}_0^2 \label{fressplit}
}
since $(I-\mathcal{P}_\text{NTK})f_\text{CK} = 0$ due to $\mathcal{P}_\text{NTK}\mathcal{P}_\text{CK} = \mathcal{P}_\text{CK}$. Combining \eqref{fsplit} and \eqref{fressplit} yields
\eqn{
\frac{\norm{f - f_\text{CK}}_0^2 }{\norm{f}_0^2} &=& \frac{ \norm{f_\text{NTK} - f_\text{CK}}_0^2 + \norm{f - f_\text{NTK}}_0^2}{\norm{f_\text{NTK}}_0^2 + \norm{f - f_\text{NTK}}_0^2} \nonumber\\
&\geq& \frac{ \norm{f_\text{NTK} - f_\text{CK}}_0^2}{\norm{f_\text{NTK}}_0^2} \nonumber
}
where we used the elementary inequality
\eqn{
\frac{\epsilon + \zeta}{\rho + \zeta} \geq \frac{\epsilon}{\rho}, \label{NDIneq}
}
for $0 < \epsilon \leq \rho$ and $\zeta \geq 0$. As a result,
\eqn{
\norm{f_\text{NTK} - f_\text{CK}}_0^2 &\leq& \left(\frac{\norm{f_\text{NTK}}_0^2}{\norm{f}_0^2}\right)\norm{f - f_\text{CK}}_0^2 \nonumber\\
&\leq& \beta^2\norm{f - f_\text{CK}}_0^2 \nonumber\\
\Leftrightarrow \norm{\mathcal{P}_\text{NTK}(f - f_\text{CK})}_0 &\leq& \beta\norm{f - f_\text{CK}}_0 \qquad (\because \mathcal{P}_\text{NTK}\mathcal{P}_\text{CK} = \mathcal{P}_\text{CK}). \nonumber
}

\end{proof}

\begin{lem}\label{LemCKNTKBd}
For any function $f \in \mathcal{F}_{\beta}$ with $\beta < 1$, 
\eqn{
\norm{f - f_\text{CK}}_0 &\leq& \frac{1}{1-\beta}\norm{f - f_\text{NTK}}_0. \nonumber
}
\end{lem} 

\begin{proof}
We have
\eqn{
\norm{f - f_\text{CK}}_0 &\leq& \norm{f - f_\text{NTK}}_0 + \norm{f_\text{NTK} - f_\text{CK}}_0 \nonumber\\
&=& \norm{f - f_\text{NTK}}_0 + \norm{\mathcal{P}_\text{NTK}(f - f_\text{CK})}_0 \qquad (\because \mathcal{P}_\text{NTK}\mathcal{P}_\text{CK} = \mathcal{P}_\text{CK}) \nonumber\\
&\leq& \norm{f - f_\text{NTK}}_0 + \beta\norm{f - f_\text{CK}}_0 \qquad (\because \text{Lemma \ref{LemProjBd}}) \nonumber\\
\Rightarrow \norm{f - f_\text{CK}}_0 &\leq& \frac{1}{1-\beta}\norm{f - f_\text{NTK}}_0. \label{TrBound2}
}
\end{proof}

For such functions, we have the following approximation results.

\begin{thm}\label{ThmMono}
Let $f \in \mathcal{F}_{\beta}$ for some $\beta < 1$ and suppose that both $(f-f_\text{CK})$ and $(f - f_\text{NTK})$ are monotonic on every sub-interval $[x_i,x_{i+1}]$ for $0 \leq i \leq N-1$. Then, there exist constants $C_1, C_2 \geq 1$ such that
\eqn{
C_1^{-1}\norm{f - f_\text{NTK}}_1 \leq \norm{f - f_\text{CK}}_1 \leq C_2\norm{f - f_\text{NTK}}_1. \label{ThmMonoRes}
}
\end{thm}

\begin{proof}\label{ProofThmMono}
We have
\eqn{
\norm{f - f_\text{NTK}}_1 &\leq& \sqrt{2}\norm{f - f_\text{NTK}}_0 \qquad (\because \text{Lemma \ref{LemMono}}) \nonumber\\
&\leq& \sqrt{2}\norm{f - f_\text{CK}}_0 \qquad (\because \eqref{TrBound}) \nonumber\\
&\leq& \sqrt{2\tau}\norm{f - f_\text{CK}}_1 \qquad (\because \eqref{NormIneq1}). \label{ThmMono1}
}

Similarly,
\eqn{
\norm{f - f_\text{CK}}_1 &\leq& \sqrt{2}\norm{f - f_\text{CK}}_0 \qquad (\because \text{Lemma \ref{LemMono}}) \nonumber\\
&\leq& \frac{\sqrt{2}}{1-\beta}\norm{f - f_\text{NTK}}_0 \qquad (\because \text{Lemma \ref{LemCKNTKBd}}) \nonumber\\
&\leq& \frac{\sqrt{2\tau}}{1-\beta}\norm{f - f_\text{NTK}}_1 \qquad (\because \eqref{NormIneq1}) . \label{ThmMono2}
}

Setting $C_1 = \sqrt{2\tau}$ and $C_2 = \frac{\sqrt{2\tau}}{1-\beta}$ and combining \eqref{ThmMono1} and \eqref{ThmMono2} yields the desired \eqref{ThmMonoRes}.

\end{proof}

Note that Theorem \ref{ThmMono} requires the residuals to be monotonic which may be too stringent. The following results rests on a more relaxed assumption. 

\begin{thm}\label{ThmLip}
Let $f \in \mathcal{F}_{\beta}$ for some $\beta < 1$ and suppose that both $(f-f_\text{CK})$ and $(f - f_\text{NTK})$ are Lipschitz, with Lipschitz constants $L_\text{CK}$ and $L_\text{NTK}$. Then, there exist constants $D_1, D_2 \geq 1$ such that
\eqn{
\norm{f - f_\text{NTK}}_1^2 \leq D_1\norm{f - f_\text{CK}}_1^2 + \frac{8(b-a)^2 L_\text{NTK}^2}{N^2} \label{ThmLipRes1}
}
and
\eqn{
\norm{f - f_\text{CK}}_1^2 \leq D_2\norm{f - f_\text{NTK}}_1^2 + \frac{8(b-a)^2 L_\text{CK}^2}{N^2}. \label{ThmLipRes2}
}
\end{thm}

\begin{proof}\label{ProofThmLip}
We have
\eqn{
\norm{f - f_\text{NTK}}_1^2 &\leq& 4\left(\norm{f - f_\text{NTK}}_0^2 + \frac{2(b-a)^2 L_\text{NTK}^2}{N^2}\right) \qquad (\because \text{Lemma \ref{LemLip}}) \nonumber\\
&\leq& 4\left(\norm{f - f_\text{CK}}_0^2 + \frac{2(b-a)^2 L_\text{NTK}^2}{N^2}\right) \qquad (\because \eqref{TrBound}) \nonumber\\
&\leq& 4\left(\tau\norm{f - f_\text{CK}}_1^2 + \frac{2(b-a)^2 L_\text{NTK}^2}{N^2}\right) \qquad (\because \eqref{NormIneq1}). \label{ThmLip1}
}

In addition,
\eqn{
\norm{f - f_\text{CK}}_1^2 &\leq& 4\left(\norm{f - f_\text{CK}}_0^2 + \frac{2(b-a)^2 L_\text{CK}^2}{N^2}\right) \qquad (\because \text{Lemma \ref{LemLip}}) \nonumber\\
&\leq& 4\left(\frac{1}{(1-\beta)^2}\norm{f - f_\text{NTK}}_0^2 + \frac{2(b-a)^2 L_\text{CK}^2}{N^2}\right) \qquad (\because \text{Lemma \ref{LemCKNTKBd}}) \nonumber\\
&\leq& 4\left(\frac{\tau}{(1-\beta)^2}\norm{f - f_\text{NTK}}_1^2 + \frac{2(b-a)^2 L_\text{CK}^2}{N^2}\right) \qquad (\because \eqref{NormIneq1}). \label{ThmLip2}
}

Setting $D_1 = 4\tau$ and $D_2 = \frac{4\tau}{(1-\beta)^2}$ yields the desired \eqref{ThmLipRes1} and \eqref{ThmLipRes2}.

\end{proof}

\section{Logistic regression}\label{SecKLMathAnalysis}

\subsection{Preliminaries}\label{SubSecKLPrelim}

In this section, we compare the performance of the two kernels under consideration for logistic regression. For brevity of exposition, we consider the problem in two dimensions. Let $\{{\bs x}_{i,j}\}_{0 \leq i \leq N_1, 0 \leq j \leq N_2} \subset Z = [a_1,b_1] \times [a_2,b_2]$ be the training points given by \eqn{
\bs{x}_{i,j} = (a_1 + i\Delta x_1, a_2 + j \Delta x_2), \nonumber
} 
where $\Delta x_1 = (b_1-a_1)/N_1$ and $\Delta x_2 = (b_2-a_2)/N_2$. We further set $h = \left(\Delta x_1^2 + \Delta x_2^2\right)^{1/2}$. The test nodes $\{{\bs y}_{i,j}\}_{0 \leq i \leq M_1, 0 \leq j \leq M_2}$ are similarly chosen as
\eqn{
{\bs y}_{i,j} = (a_1 + i\Delta y_1, a_2 + j \Delta y_2), \nonumber
} 
where $\Delta y_1 = (b_1-a_1)/M_1$ and $\Delta y_2 = (b_2-a_2)/M_2$. We also assume that $\tau_k = M_k/N_k$ is an integer greater than one for $k = 1,2$. 

Following Section \ref{SecPrelim}, each training and test point is assigned a class labelled $0$ or $1$, and denoted respectively by $\chi_{ij}$ for every ${0 \leq i \leq N_1, 0 \leq j \leq N_2}$ and $\mu_{ij}$ for each ${0 \leq i \leq M_1, 0 \leq j \leq M_2}$. Given a kernel function $\text{ker}: Z \times Z \to \mathbb{R}$, we can define the kernel matrix 
\eqn{
H_\text{ker} = (\text{ker}(\bs{x}_{i,j} , {\bs{x}_{k,l}}))_{0 \leq i,k \leq N_1, 0 \leq j,l \leq N_2} \in \mathbb{R}^{N_1N_2 \times N_1N_2}, \nonumber
} 
with the indices arranged suitably. 

We note first that the NTK must necessarily outperform the CK on the training points, i.e.,
\eqn{
\ell_{0,\text{NTK}}\left(\bs{\alpha}^{*,\text{NTK}}\right) \leq \ell_{0,\text{CK}}\left(\bs{\alpha}^{*,\text{CK}}\right). \label{LogTrBound}
} 

This is most easily seen by rewriting \eqref{LogTrBound} in terms of the feature maps: since
\eqn{
\ell_{0,\text{CK}}\left(\bs{\alpha}^{*,\text{CK}}\right) = \widetilde{\ell}_{0,\text{CK}}\left(\bs{\beta}^{*,\text{CK}}\right), \nonumber
}
from \eqref{KLEquivLosses}, and $\Phi_\text{NTK} = \left(\Phi_\text{CK} \ \ \Phi_\text{E}\right)^\top$, setting $\bs{\gamma} = \left(\bs{\beta}^{*,\text{CK}} \ \ \bs{0}\right)^\top \in \mathbb{R}^{|\theta|}$ yields
\eqn{
\ell_{0,\text{CK}}\left(\bs{\alpha}^{*,\text{CK}}\right) = \widetilde{\ell}_{0,\text{CK}}\left(\bs{\beta}^{*,\text{CK}}\right) = \widetilde{\ell}_{0,\text{NTK}}\left(\bs{\gamma}\right) \geq \widetilde{\ell}_{0,\text{NTK}}\left(\bs{\beta}^{*,\text{NTK}}\right) = \ell_{0,\text{NTK}}\left(\bs{\alpha}^{*,\text{NTK}}\right). \label{LogTrBound0}
} 

In order to show that an equivalence result in the manner of Theorem \ref{ThmMono} holds for the test errors, we follow a similar strategy to that for the function regression problem in Section \ref{SecKRMathAnalysis}: 
\begin{itemize}
\item[(a)] identify two cases that permit a transfer of information from training to test nodes;

\item[(b)] extend \eqref{LogTrBound} to an equivalence result on the training nodes for the two kernels
\end{itemize}

These ingredients can then be combined to yield the desired conclusion.

\subsection{Another pair of lemmas}\label{SubSecLogTwoLemma}

We begin by noting that
\eqn{
\ell_{0,\text{ker}}(\bs{\alpha}^{*,\text{ker}}) \leq \ell_{1,\text{ker}}(\bs{\alpha}^{*,\text{ker}}) \label{LossIneq1}
}
simply by virtue of the fact that every training point is also a test point. To complete (a), we need to bound the error on the finer grid by that on the coarser grid. The quantity of interest is the linear combination of the feature map components
\eqn{
\psi_\text{ker}(\bs{z}) = \sum_{m = 1}^{|\theta|} \Phi_{\text{ker},m}(\bs{z})\bs{\beta}^{*,\text{ker}}_m \label{PsiDefn}
}
as it determines the size of the gap between $\bs{z} \in Z$ and the separating hyperplane in the feature space. However, we also need to keep track of the class label that should be associated with $\bs{z}$ as that determines how it shows up in the loss function \eqref{KLEqTrLoss}. Let $\eta:Z \to \{0,1\}$ be the true class label for every point (so that, e.g., $\chi_{ij} = \eta(\bs{x}_{i,j})$ and $\mu_{ij} = \eta(\bs{y}_{i,j})$), and set
\eqn{
\widehat{\psi}_\text{ker}(\bs{z}) = \left(1-2\eta(\bs{z})\right)\psi_\text{ker}(\bs{z}). \label{PsiHatDefn}
}

Using this function, we can write the minimized training and test losses as simply
\eqn{
\widetilde{\ell}_{0,\text{ker}} = \sum_{i = 0}^{N_1}\sum_{j = 0}^{N_2} \ln\left[1 + \exp\left(\widehat{\psi}_\text{ker}(\bs{x}_{i,j})\right)\right] \label{TrLoss2Psi}
}
and 
\eqn{
\widetilde{\ell}_{1,\text{ker}} = \sum_{i = 0}^{M_1}\sum_{j = 0}^{M_2} \ln\left[1 + \exp\left(\widehat{\psi}_\text{ker}(\bs{y}_{i,j})\right)\right], \label{TeLoss2Psi}
}
where we have suppressed the dependence on $\bs{\beta}^{*,\text{ker}}$ for clarity of exposition. For any $0 \leq i \leq N_1-1$ and $0 \leq j \leq N_2-1$, we also define
\eqn{
\Lambda_{ij} = \{\bs{x}_{i,j},\bs{x}_{i+1,j},\bs{x}_{i,j+1},\bs{x}_{i+1,j+1}\} \nonumber
}
to make it easier to navigate the training grid.

The two regimes that interest us are a corner maximum condition and Lipschitz continuity. We make these precise in the following lemmas.

\begin{lem}\label{LemLogCorMax}
Suppose that on every rectangle with corners $\Lambda_{ij}$, the function $\widehat{\psi}_\text{ker}$ achieves its maximum value on $\Lambda_{ij}$. Then,
\eqn{
\widetilde{\ell}_{1,\text{ker}}  \leq \tau_1\tau_2\widetilde{\ell}_{0,\text{ker}}, \label{LogCorMax} 
}
\end{lem}

\begin{proof}
Every test point $\bs{y}_{k,l}$ belongs to some rectangle with corners $\Lambda_{k'l'}$. The contribution from this point to the test loss is then dominated by the contributions from the corner points, i.e.,
\eqn{
\ln\left[1 + \exp\left(\widehat{\psi}_\text{ker}(\bs{y}_{k,l})\right)\right] &\leq& \max_{k'\leq i \leq k'+1, l' \leq j \leq l'+1} \ln\left[1 + \exp\left(\widehat{\psi}_\text{ker}(\bs{x}_{i,j})\right)\right] \label{LemLogCorMax1},
}
with the result that
\eqn{
\sum_{k = 0}^{M_1}\sum_{l = 0}^{M_2} \ln\left[1 + \exp\left(\widehat{\psi}_\text{ker}(\bs{y}_{k,l})\right)\right] &\leq& \left(\frac{M_1}{N_1}\right)\left(\frac{M_2}{N_2}\right)\sum_{i = 0}^{N_1}\sum_{j = 0}^{N_2} \ln\left[1 + \exp\left(\widehat{\psi}_\text{ker}(\bs{x}_{i,j})\right)\right], \nonumber 
}
and hence
\eqn{
\widetilde{\ell}_{1,\text{ker}}  &\leq& \tau_1\tau_2\widetilde{\ell}_{0,\text{ker}}. \nonumber 
}

\end{proof}

We say that the labelling function $\eta$ possesses the {\it matching property} with respect to the given training and test nodes if every test point $\bs{y}_{k,l}$ is contained in an enclosing rectangle with corners $\Lambda_{k'l'}$ such that $\eta$ agrees with $\eta(\bs{y}_{k,l})$ on at least one corner point. In other words, we do not have the matching property if and only if there exists some test point $\bs{y}_{k,l}$ such that, for any rectangle with corners $\Lambda_{k'l'}$ that may enclose it, we have $\eta(\bs{y}_{k,l}) \notin \eta \left(\Lambda_{k'l'}\right)$.

\begin{lem}\label{LemLogLip}
Suppose that $\psi_\text{ker}$ is Lipschitz continuous on $D$, with Lipschitz constant $L_\text{ker}$, and $\eta$ possesses the matching property. Then, 
\eqn{
\widetilde{\ell}_{1,\text{ker}}  &\leq& \exp(hL_\text{ker})\tau_1\tau_2\widetilde{\ell}_{0,\text{ker}}. \label{LogLip}
}
\end{lem}
\begin{proof}
For any test point $\bs{y}_{k,l}$, let $\Lambda_{k'l'}$ contain the corners of the enclosing rectangle that respect the matching property. Without loss of generality, we can assume that $\eta(\bs{y}_{k,l}) = \eta(\bs{x}_{k',l'}) = 0$.  We have
\eqn{
\psi_\text{ker}\left(\bs{y}_{k,l}\right) \leq \psi_\text{ker}\left(\bs{x}_{k',l'}\right) + hL_\text{ker}, \nonumber
}
and hence
\eqn{
\ln\left[1 + \exp\left(\psi_\text{ker}(\bs{y}_{k,l})\right)\right] &\leq& \ln\left[1 + \exp\left(\psi_\text{ker}\left(\bs{x}_{k',l'}\right)\right) \exp\left(hL_\text{ker}\right)\right] \nonumber\\
&\leq& \exp\left(hL_\text{ker}\right)\ln\left[1 + \exp\left(\psi_\text{ker}\left(\bs{x}_{k',l'}\right)\right)\right], \nonumber
}
where we used the fact that
\eqn{
\ln\left(1 + \rho x\right) \leq \rho\ln(1 + x) \label{LogIneq}
}
whenever $\rho \geq 1$ and $x \geq 0$. In the case $\eta(\bs{y}_{k,l}) = \eta(\bs{x}_{k',l'}) = 1$, we similarly obtain
\eqn{
\ln\left[1 + \exp\left(-\psi_\text{ker}(\bs{y}_{k,l})\right)\right] &\leq& \exp\left(hL_\text{ker}\right)\ln\left[1 + \exp\left(-\psi_\text{ker}\left(\bs{x}_{k',l'}\right)\right)\right], \nonumber
}
and, consequently,
\eqn{
\sum_{k = 0}^{M_1}\sum_{l = 0}^{M_2} \ln\left[1 + \exp\left(\widehat{\psi}_\text{ker}(\bs{y}_{k,l})\right)\right] &\leq& \exp\left(hL_\text{ker}\right)\left(\frac{M_1}{N_1}\right)\left(\frac{M_2}{N_2}\right)\sum_{i = 0}^{N_1}\sum_{j = 0}^{N_2} \ln\left[1 + \exp\left(\widehat{\psi}_\text{ker}(\bs{x}_{i,j})\right)\right] \nonumber
}
so that
\eqn{
\widetilde{\ell}_{1,\text{ker}}  &\leq& \exp\left(hL_\text{ker}\right)\tau_1\tau_2\widetilde{\ell}_{0,\text{ker}}. \nonumber 
}

\end{proof}

In combination with \eqref{LossIneq1}, Lemmas \eqref{LemLogCorMax} and \eqref{LemLogLip} allow us to move seamlessly between training and test losses. Note that the assumptions underlying these lemmas can only conceivably be satisfied by smooth $\psi_\text{ker}$, thus almost completely ruling out the suitability of ReLU activations. On the other hand, the Tanh activation function equips the kernels with Lipschitz continuity by default; the matching property is a technical assumption to rule out pathological cases where the label of a test point disagrees with those of all of its neighbouring training points.

\subsection{Performance comparison}\label{SubSecLogPerfComp}

Next, we address task (b) outlined at the end of Subsection \ref{SubSecKLPrelim}. Note that for any $0 \leq i \leq N_1$, $0 \leq j \leq N_2$, we have
\eqn{
\ln\left[1 + \exp\left(\widehat{\psi}_\text{CK}\left(\bs{x}_{i,j}\right)\right)\right] &=& \ln\left[1 + \exp\left(\widehat{\psi}_\text{CK}\left(\bs{x}_{i,j}\right) - \widehat{\psi}_\text{NTK}\left(\bs{x}_{i,j}\right)\right)\exp\left(\widehat{\psi}_\text{NTK}\left(\bs{x}_{i,j}\right)\right)\right] \nonumber\\
&\leq& \max\left\{1,\exp\left(\widehat{\psi}_\text{CK}\left(\bs{x}_{i,j}\right) - \widehat{\psi}_\text{NTK}\left(\bs{x}_{i,j}\right)\right)\right\}\ln\left[1 + \exp\left(\widehat{\psi}_\text{NTK}\left(\bs{x}_{i,j}\right)\right)\right] \nonumber
}
where we used \eqref{NDIneq} again. Set
\eqn{
\omega &=& \max_{0 \leq i \leq N_1, 0 \leq j \leq N_2} \exp\left(\widehat{\psi}_\text{CK}\left(\bs{x}_{i,j}\right) - \widehat{\psi}_\text{NTK}\left(\bs{x}_{i,j}\right)\right); \label{deltadefn}
}
it follows from \eqref{LogTrBound} that $\omega \geq 1$. We can therefore write
\eqn{
\widetilde{\ell}_{0,\text{CK}} \leq \omega \widetilde{\ell}_{0,\text{NTK}}. \label{LogTrBound2} 
}

Combining \eqref{LogTrBound} and \eqref{LogTrBound2} with \eqref{LossIneq1} and Lemmas \ref{LemLogCorMax} and \ref{LemLogLip} then yields the following two results.

\begin{thm}\label{ThmLogCorMax}
Suppose that on every rectangle with corners $\Lambda_{ij}$, the functions $\widehat{\psi}_\text{NTK}$ and $\widehat{\psi}_\text{CK}$ achieve their maximum values on $\Lambda_{ij}$. Then, there exist constants $C_1,C_2 \geq 1$ such that
\eqn{
C_1^{-1}\ell_{1,\text{NTK}}  \leq \ell_{1,\text{CK}}  \leq C_2 \ell_{1,\text{NTK}} \label{ThmCorMaxRes}
}
\end{thm}
\begin{proof}
We have
\eqn{
\ell_{1,\text{NTK}} &\leq& (\tau_1\tau_2)\ell_{0,\text{NTK}} \qquad (\because \text{Lemma }\ref{LemLogCorMax}) \nonumber\\
&\leq& (\tau_1\tau_2)\ell_{0,\text{CK}} \qquad (\because \eqref{LogTrBound}) \nonumber\\
&\leq& (\tau_1\tau_2)\ell_{1,\text{CK}} \qquad (\because \eqref{LossIneq1}). \label{ThmCorMax1}
}

Similarly,
\eqn{
\ell_{1,\text{CK}} &\leq& (\tau_1\tau_2) \ell_{0,\text{CK}} \qquad (\because \text{Lemma }\ref{LemLogCorMax}) \nonumber\\
&\leq& (\tau_1\tau_2\omega) \ell_{0,\text{NTK}} \qquad (\because \eqref{LogTrBound2}) \nonumber\\	
&\leq& (\tau_1\tau_2\omega) \ell_{1,\text{NTK}} \qquad (\because \eqref{LossIneq1}). \label{ThmCorMax2}
}

Setting $C_1 = \tau_1\tau_2$ and $C_2 = \tau_1\tau_2\omega$ and combining \eqref{ThmCorMax1} and \eqref{ThmCorMax2} yields \eqref{ThmCorMaxRes}.

\end{proof}

\begin{thm}\label{ThmLogLip}
Suppose that both $\psi_\text{NTK}$ and $\psi_\text{CK}$ are Lipschitz continuous on $D$ and $\eta$ possesses the matching property. Then, there exist constants $D_1,D_2 \geq 1$ such that
\eqn{
D_1^{-1}\ell_{1,\text{NTK}}  \leq \ell_{1,\text{CK}}  \leq D_2 \ell_{1,\text{NTK}} \label{ThmLogLipRes}
}
\end{thm}
\begin{proof}
As before, let $L_\text{ker}$ be the Lipschitz constant of $\psi_\text{ker}$ for both the kernels. We then have
\eqn{
\ell_{1,\text{NTK}} &\leq& e^{hL_\text{NTK}}\ell_{0,\text{NTK}} \qquad (\because \text{Lemma }\ref{LemLogLip}) \nonumber\\
&\leq& e^{hL_\text{NTK}}\ell_{0,\text{CK}} \qquad (\because \eqref{LogTrBound}) \nonumber\\
&\leq& e^{hL_\text{NTK}}\ell_{1,\text{CK}} \qquad (\because \eqref{LossIneq1}). \label{ThmLogLip1}
}

Similarly,
\eqn{
\ell_{1,\text{CK}} &\leq& e^{hL_\text{CK}} \ell_{0,\text{CK}} \qquad (\because \text{Lemma }\ref{LemLogLip}) \nonumber\\
&\leq& e^{hL_\text{CK}} \ell_{0,\text{NTK}} \qquad (\because \eqref{LogTrBound2}) \nonumber\\	
&\leq& \omega e^{hL_\text{CK}} \ell_{1,\text{NTK}} \qquad (\because \eqref{LossIneq1}). \label{ThmLogLip2}
}

Setting $D_1 = e^{hL_\text{NTK}}$ and $D_2 = \omega e^{hL_\text{CK}}$ and combining \eqref{ThmLogLip1} and \eqref{ThmLogLip2} yields \eqref{ThmLogLipRes}.

\end{proof}

\def \wida{0.8}
\def \widb{0.48}
\def \widcc{0.31}
\def \widc{0.43}

\section{Numerical tests}\label{SecNumTests}

In this section, we present numerical evidence to support the theoretical findings established in the Sections \ref{SecKRMathAnalysis} and \ref{SecKLMathAnalysis}. In addition, we identify myriad benefits of employing CK approximations over those coming from NN and NTK, including better conditioning, inexpensive gains in accuracy, and improved robustness. 

\subsection{Illustrative examples for regression and classification}
We first consider the approximation problem for smooth functions $f:[-1,1] \to \mathbb{R}$. As examples, we use
\eqn{
f_1(x) = e^{\sin(2\pi x)}, \qquad f_2(x) = e^{3x}, \qquad f_3(x) = \cos\left(e^{3x}\right). \label{KRExamples}
} 

We note that $f_3$ oscillates rapidly close to $x = 1$ so it poses a particularly stiff challenge to NN and kernel approximators alike. 

\begin{figure}[tbph]
\centering
\subfigure[$f_1(x) = e^{\sin(2\pi x)}$]
{\includegraphics[width=\widc\textwidth]{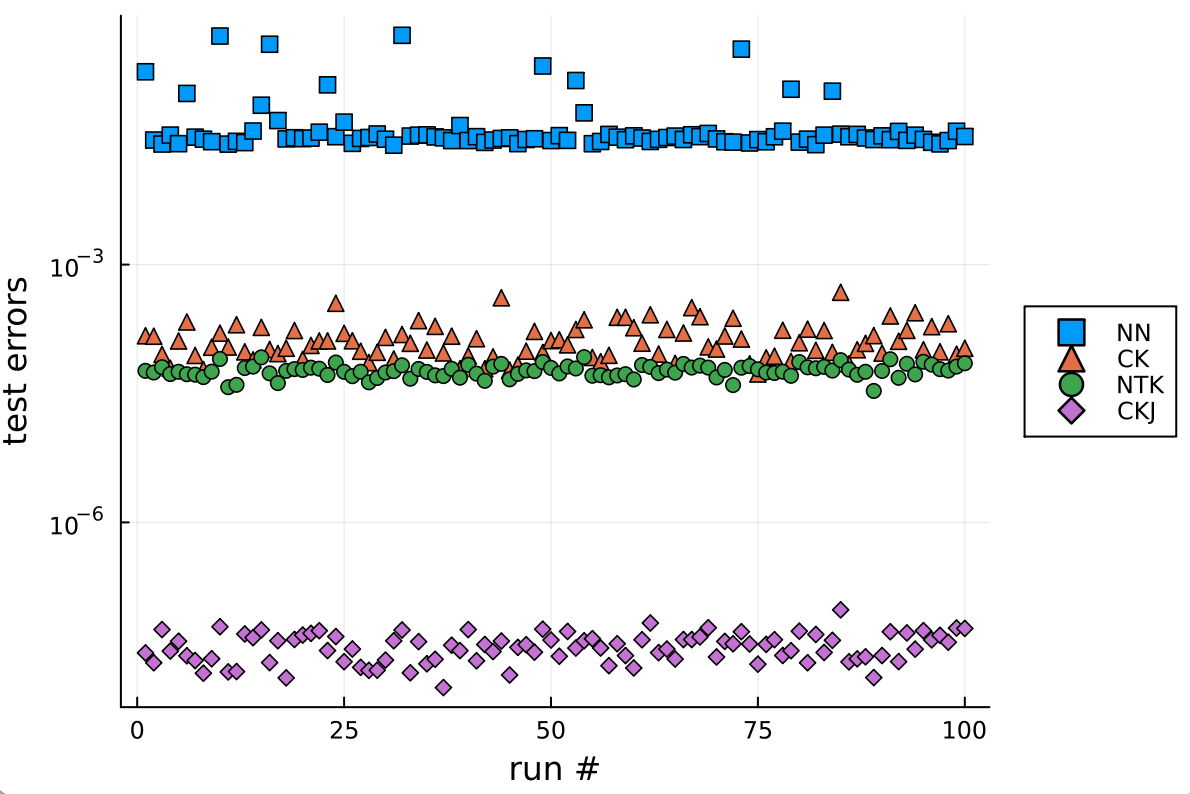}
\label{SFigKR1Tanh}
}
\subfigure[$f_2(x) = e^{3x}$]
{\includegraphics[width=\widc\textwidth]{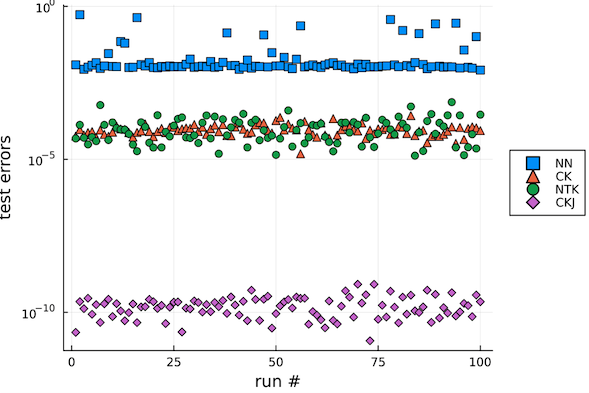}
\label{SFigKR2Tanh}
}
\subfigure[$f_3(x) = \cos\left(e^{3x}\right)$]
{\includegraphics[width=\widc\textwidth]{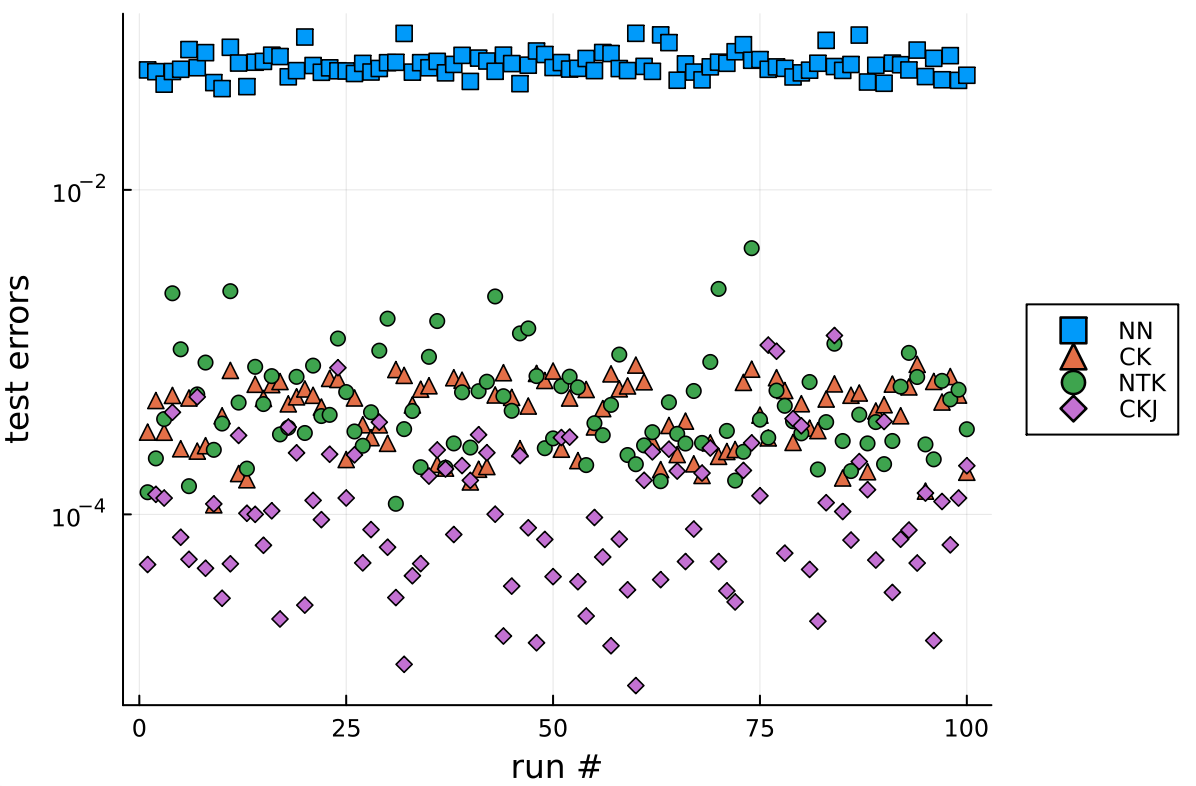}
\label{SFigKR3Tanh}
}
\caption{The results of function regression for 100 different NNs, trained for 2400 epochs, and the corresponding approximations using the NTK, CK, and CKJ extracted from the NN at the end of the training.}\label{FigKRTanh}
\end{figure}

The sizes of the training and testing grids, as defined in Subsection \ref{SubSecKRPrelim}, are fixed at $N = 200$ and $M = 600$ respectively, so $\tau = 3$. We train NNs of the form \eqref{NN} with $L = 3$ and $d_l = d = 128$ for $1 \leq l \leq 3$ for $2400$ epochs using the loss function defined in \eqref{RSqLoss} and  the weights given in \eqref{KRWts}. We employ the ADAM optimizer with the learning rate set at $10^{-3}$ and primarily use Tanh as the activation function.

After training an NN, we assemble the CK and NTK and compute the corresponding kernel approximations \eqref{KRApprox}. The details of the algorithm used for extracting the NTK are given in Appendix \ref{AppNTKComp}. The CK is easily obtained from the values of the last hidden layer and using \eqref{CKDefn}. Since the feature space dimension for the CK is $(d+1)$, assembling and using the corresponding Jacobian is also inexpensive. The results from using the Jacobian are indicated by CKJ in the following plots; a recipe for using this approximation is detailed in Appendix \ref{AppBFs}.

\begin{figure}[tbph]
\centering
\subfigure[$f_1(x) = e^{\sin(2\pi x)}$]
{\includegraphics[width=\widcc\textwidth]{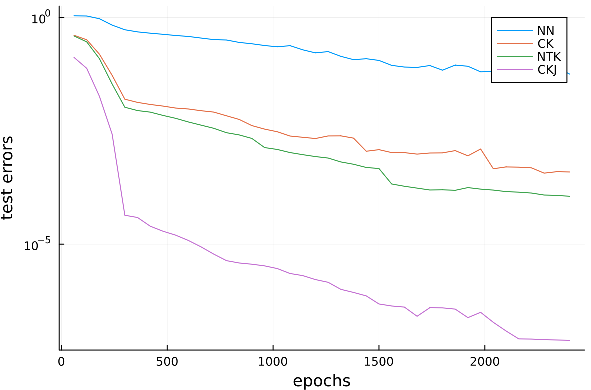}
\label{SFigKR1TanhEpochs}
}
\subfigure[$f_2(x) = e^{3x}$]
{\includegraphics[width=\widcc\textwidth]{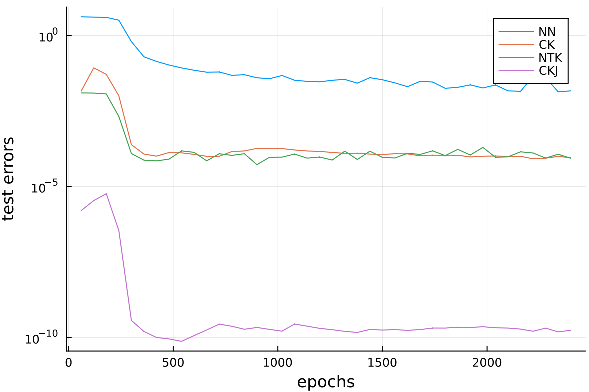}
\label{SFigKR2TanhEpochs}
}
\subfigure[$f_3(x) = \cos\left(e^{3x}\right)$]
{\includegraphics[width=\widcc\textwidth]{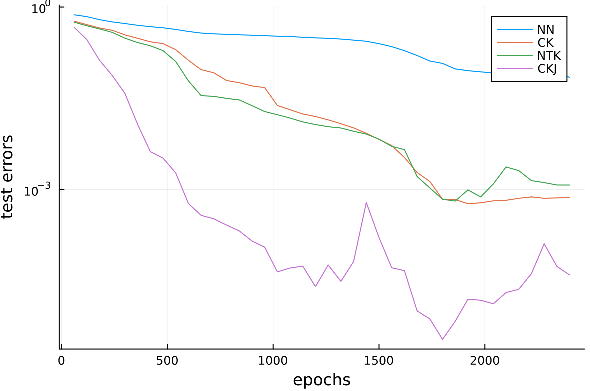}
\label{SFigKR3TanhEpochs}
}
\caption{The test errors for NNs over the course of being trained to approximate the given functions, and for the corresponding NTK, CK, and CKJ approximations. The errors are averaged over ten iterations to reduce the effects of random initialization.}\label{FigKRTanhEpochs}
\end{figure}

In Figure \ref{FigKRTanh}, we show the test errors for the example target functions for 100 runs of the procedure described above. Owing to randomness in the initialization, the results vary over iterations; nevertheless, a clear separation is evident between the NN test errors and those of the two kernel approximations. More importantly, we note that the latter pair are fairly similar so that, while one may appear to be superior over the other for a particular example (e.g., the NTK for $f_1$), the difference is only marginal compared to the supremacy both enjoy over the NN.  

Figure \ref{FigKRTanh} also shows that the CK Jacobian yields test errors that improve on the kernel approximations by several orders of magnitude. This is a consequence of the much better conditioning possessed by the former: assembling the kernel matrix $\widehat{H}_\text{ker}$ squares the singular values and applying its pseudo-inverse in \eqref{KRApprox} inflates round-off errors more than using the Jacobian does, with the result that the test errors plateau earlier. We emphasize that the use of the Jacobian is only afforded by the CK since the dimension of the feature space corresponding to the NTK is prohibitively large. 

\begin{figure}[tbph]
\centering
\subfigure[$f_1(x) = e^{\sin(2\pi x)}$]
{\includegraphics[width=\widc\textwidth]{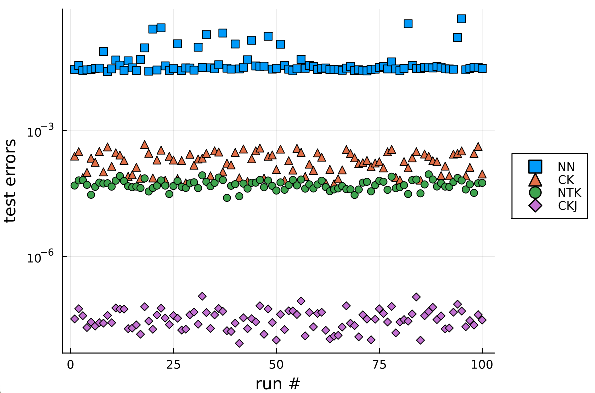}
\label{SFigKR1Tanh_256}
}
\subfigure[$f_2(x) = e^{3x}$]
{\includegraphics[width=\widc\textwidth]{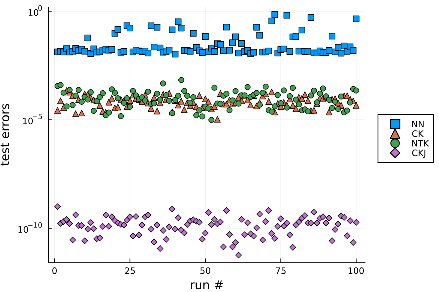}
\label{SFigKR2Tanh_256}
}
\subfigure[$f_3(x) = \cos\left(e^{3x}\right)$]
{\includegraphics[width=\widc\textwidth]{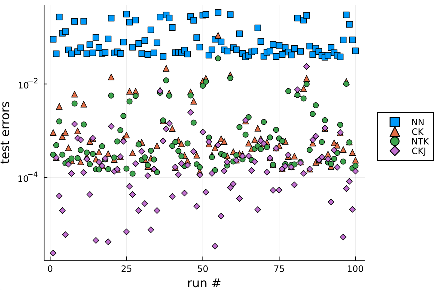}
\label{SFigKR3Tanh_256}
}
\caption{Function regression results for 100 NNs with the widths of the hidden layers set to 256, and the corresponding NTK, CK, and CKJ approximations.}\label{FigKRTanh_256}
\end{figure}

\begin{figure}[tbph]
\centering
\subfigure[$f_1(x) = e^{\sin(2\pi x)}$]
{\includegraphics[width=\widc\textwidth]{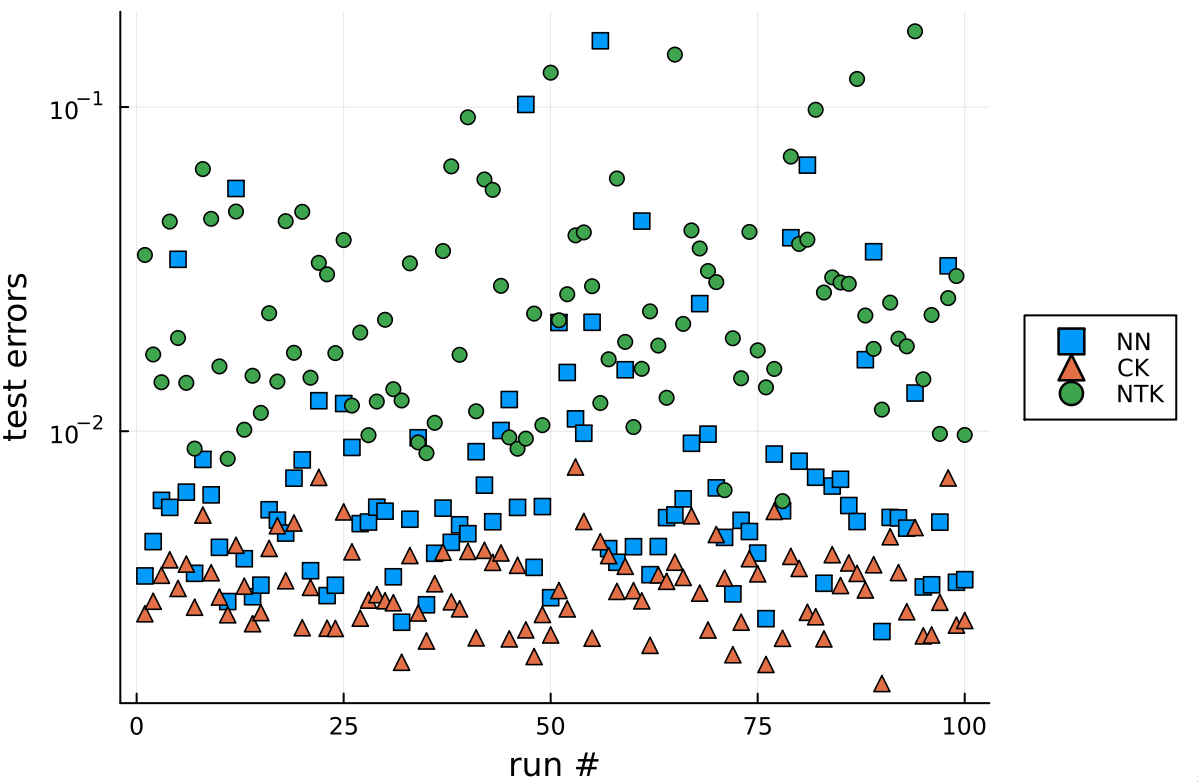}
\label{SFigKR1Relu}
}
\subfigure[$f_2(x) = e^{3x}$]
{\includegraphics[width=\widc\textwidth]{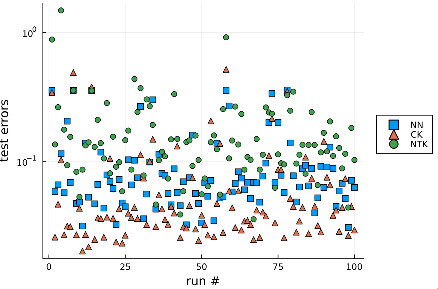}
\label{SFigKR2Relu}
}
\subfigure[$f_3(x) = \cos\left(e^{3x}\right)$]
{\includegraphics[width=\widc\textwidth]{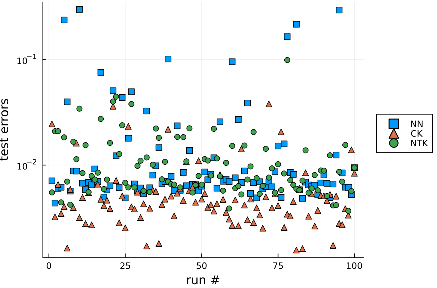}
\label{SFigKR3Relu}
}
\caption{The test errors for function regression for 100 trained NNs using the ReLU activation function, and the corresponding NTK and CK approximations.}\label{FigKRRelu}
\end{figure}

\begin{figure}[tbph]
\centering
\subfigure[$f_1(x) = e^{\sin(2\pi x)}$]
{\includegraphics[width=\widcc\textwidth]{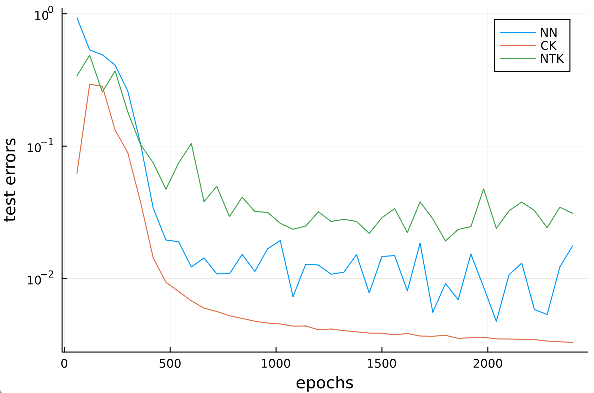}
\label{SFigKR1ReluEpochs}
}
\subfigure[$f_2(x) = e^{3x}$]
{\includegraphics[width=\widcc\textwidth]{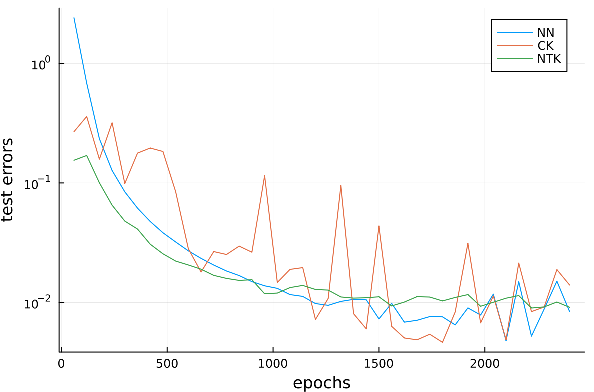}
\label{SFigKR2ReluEpochs}
}
\subfigure[$f_3(x) = \cos\left(e^{3x}\right)$]
{\includegraphics[width=\widcc\textwidth]{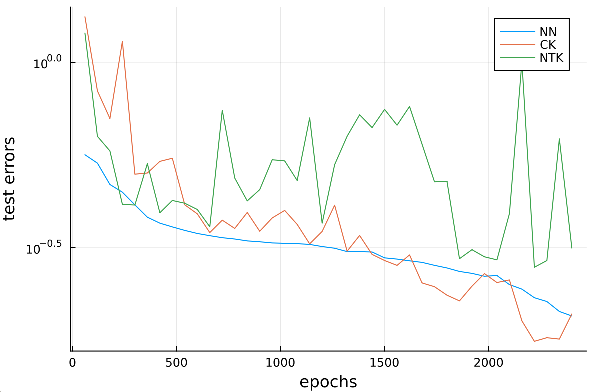}
\label{SFigKR3ReluEpochs}
}
\caption{Averaged test errors for function regression over the course of training ten NNs using ReLU activations, and the corresponding NTK and CK approximations.}\label{FigKRReluEpochs}
\end{figure}

These claims are further supported by the evolution of the errors over the course of the training, shown in Figure \ref{FigKRTanhEpochs}. The test errors are computed by averaging over ten iterations to reduce the effect of random initialization. While the NN errors can be seen to decay gradually, the kernel approximations do so more rapidly, particularly over the first few hundred epochs. The close resemblance between the two error profiles is noteworthy and supports our contention that the CK can serve as a proxy for the NTK. Finally, we note that, as expected, the CKJ errors decay even faster and level off at a threshold several orders of magnitude lower.

In Figure \ref{FigKRTanh_256}, we show that these conclusions are also valid when we repeat the experiments with the width of the hidden layers in the NN set to $d = 256$. In particular, this assuages any concerns that the beneficial properties enjoyed by the CK and CKJ approximations may be a consequence of solving an overdetermined least-squares problem: with the larger width, this is no longer the case, and the similarity in the corresponding error plots in Figures \ref{FigKRTanh} and \ref{FigKRTanh_256} establishes that our assertions are independent of NN width.

The same, however, cannot be said about the activation functions. Figure \ref{FigKRRelu} shows that using ReLU activations instead of Tanh leads to the erasure of the error segregation seen in the earlier diagrams. In addition, the NTK approximators paradoxically appear to perform the worst, with the CK results only slightly bettering the NN ones. Figure \ref{FigKRReluEpochs} highlights that this is generally the case over the course of the NN training as well. These observations can be explained as a combined effect of the over-parameterization provided by the NTK and $\mathcal{S}_\text{NTK}$ containing discontinuous functions (due to the derivatives of the ReLU activations; see Appendix \ref{AppNTKComp}, and \eqref{NTKcont} in particular, for more details). The resulting approximator manages to fit the training data very well but struggles with generalizing to the test data-points. In contrast, the $\mathcal{S}_\text{CK}$ consists only of continuous (albeit non-differentiable) functions, with the result that, while the approximation accuracy is lower than we observed with Tanh, it does not suffer from over-fitting. This can be regarded as indicative of the dangers associated with over-parameterization, the care that must be taken with ReLU-based approximators on unseen data, and the robustness possessed by the CK.

Next, we present the results for binary logistic regression on $Z = [-1,1]^2$. We set $N_1 = 11$, $N_2 = 7$, $M_1 = 22$ and $M_2 = 21$ (so $\tau_1 = 2$ and $\tau_2 = 3$) to specify the uniform training and testing grids (see Subsection \ref{SubSecKLPrelim} for the details). The labels are generated with respect to a known separating boundary $F(x_1,x_2) = 0$ for some $F:Z \to \mathbb{R}$, so the labelling function is $\eta(x_1,x_2) = \left(\text{sign}\left(F(x_1,x_2)\right)+1\right)/2$; we consider two examples:
\eqn{
F_1(x_1,x_2) = 4x_1^2 - 3x_1 + 5x_2 - 1, \quad F_2(x_1,x_2) = 2x_1^3 - 0.6x_1^2 - 1.94x_1 + x_2  + 0.2. \label{KLExamples}
}

\begin{figure}[tbph]
\centering
\subfigure[$F_1(x_1,x_2) = 4x_1^2 - 3x_1 + 5x_2 - 1 = 0$]
{{\includegraphics[width=\widb\textwidth]{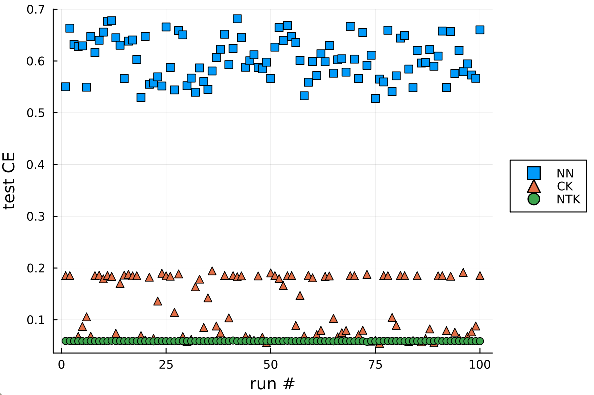}
\label{SFigKL1Tanh}
}
{\includegraphics[width=\widb\textwidth]{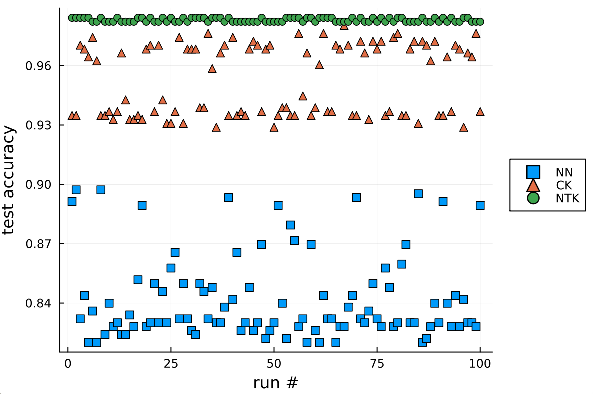}
\label{SFigKL1TanhAcc}
}}
\subfigure[$F_2(x_1,x_2) = 2x_1^3 - 0.6x_1^2 - 1.94x_1 + x_2  + 0.2 = 0$]
{{\includegraphics[width=\widb\textwidth]{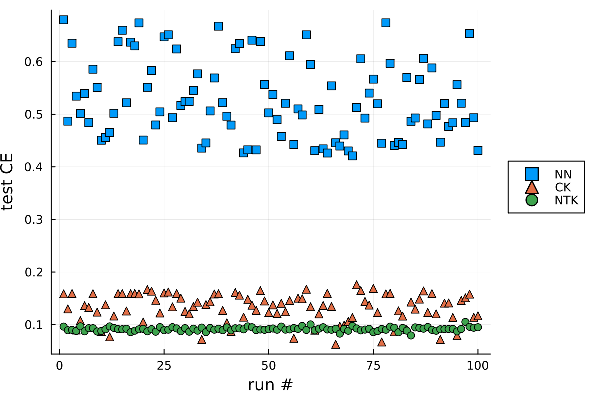}
\label{SFigKL2Tanh}
}
{\includegraphics[width=\widb\textwidth]{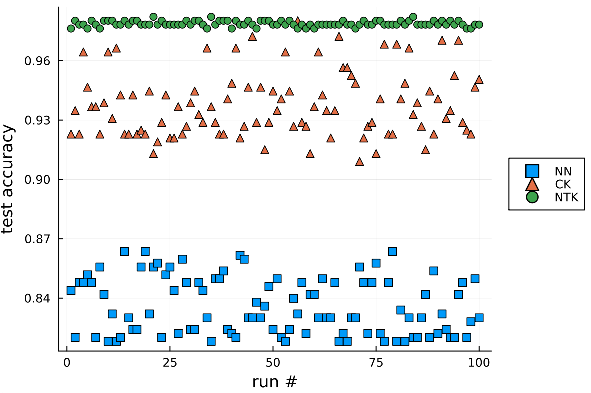}
\label{SFigKL2TanhAcc}
}}
\caption{The test cross-entropies and accuracies for logistic regression performed with 100 NNs, and the corresponding NTK and CK results.}\label{FigKLTanh}
\end{figure}

\begin{figure}[tbph]
\centering
\subfigure[$F_1(x_1,x_2) = 4x_1^2 - 3x_1 + 5x_2 - 1 = 0$]
{{\includegraphics[width=\widb\textwidth]{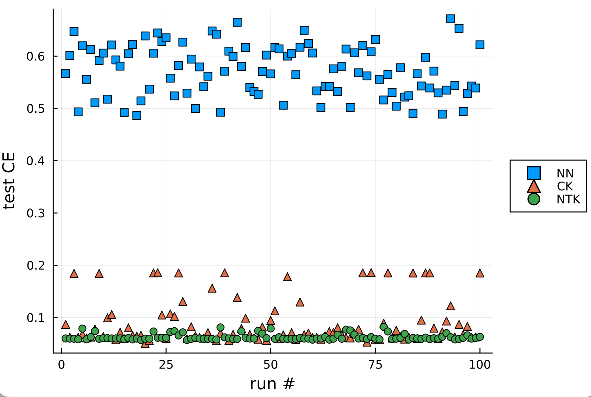}
\label{SFigKL1Tanh64}
}
{\includegraphics[width=\widb\textwidth]{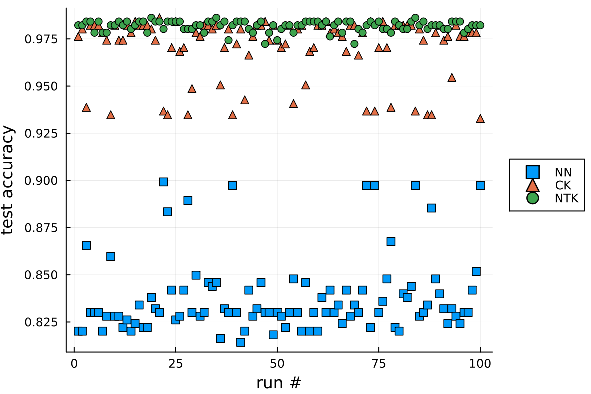}
\label{SFigKL1Tanh64Acc}
}}
\subfigure[$F_2(x_1,x_2) = 2x_1^3 - 0.6x_1^2 - 1.94x_1 + x_2  + 0.2 = 0$]
{{\includegraphics[width=\widb\textwidth]{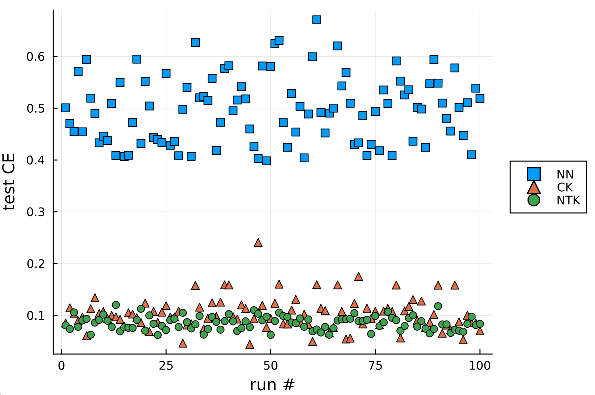}
\label{SFigKL2Tanh64}
}
{\includegraphics[width=\widb\textwidth]{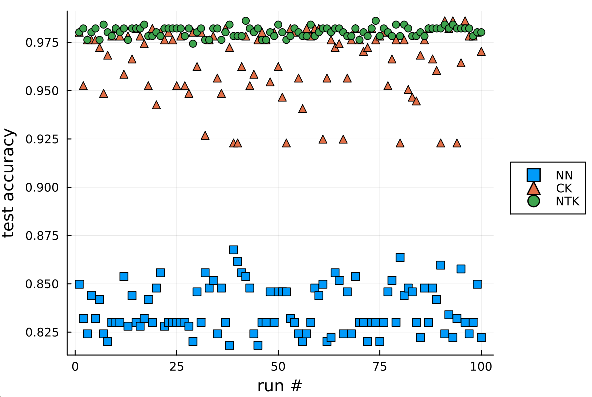}
\label{SFigKL2Tanh64Acc}
}}
\caption{The cross-entropies and accuracies for the test dataset for logistic regression performed with 100 NNs with hidden layer widths set to 64, and the corresponding NTK and CK results.}\label{FigKLTanh64}
\end{figure}

\begin{figure}[tbph]
\centering
\subfigure[$F_1(x_1,x_2) = 4x_1^2 - 3x_1 + 5x_2 - 1 = 0$]
{{\includegraphics[width=\widb\textwidth]{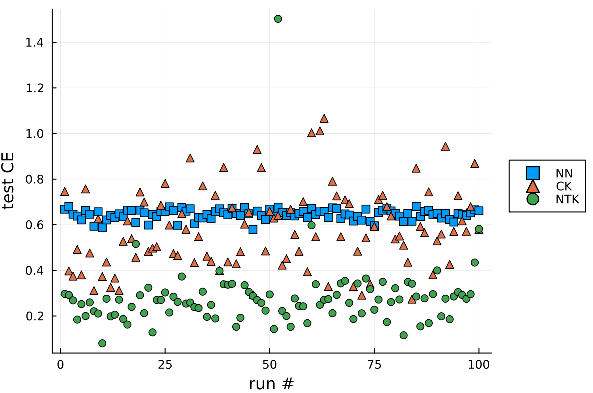}
\label{SFigKL1Relu}
}
{\includegraphics[width=\widb\textwidth]{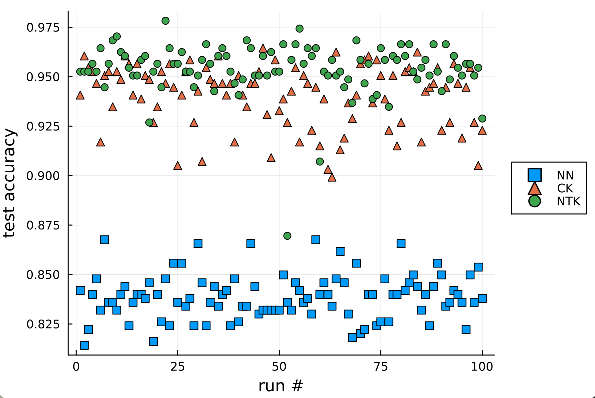}
\label{SFigKL1ReluAcc}
}}
\subfigure[$F_2(x_1,x_2) = 2x_1^3 - 0.6x_1^2 - 1.94x_1 + x_2  + 0.2 = 0$]
{{\includegraphics[width=\widb\textwidth]{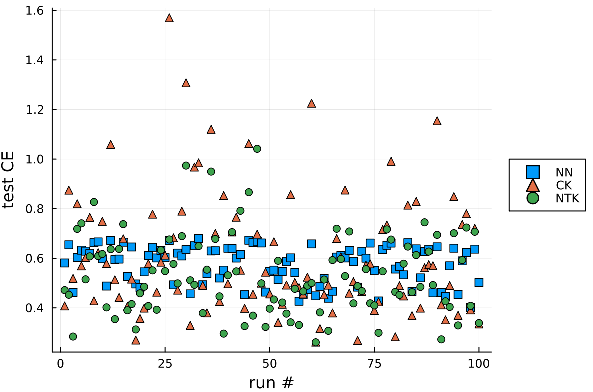}
\label{SFigKL2Relu}
}
{\includegraphics[width=\widb\textwidth]{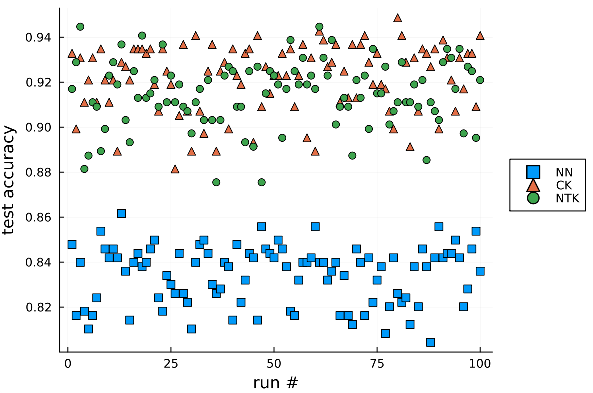}
\label{SFigKL2ReluAcc}
}}
\caption{Test results for 100 NNs using the ReLU activation function for logistic regression, and the corresponding NTK and CK metrics.}\label{FigKLRelu}
\end{figure}

\begin{figure}[tbph]
\centering
\subfigure[$F_1(x_1,x_2) = 4x_1^2 - 3x_1 + 5x_2 - 1 = 0$]
{{\includegraphics[width=\widb\textwidth]{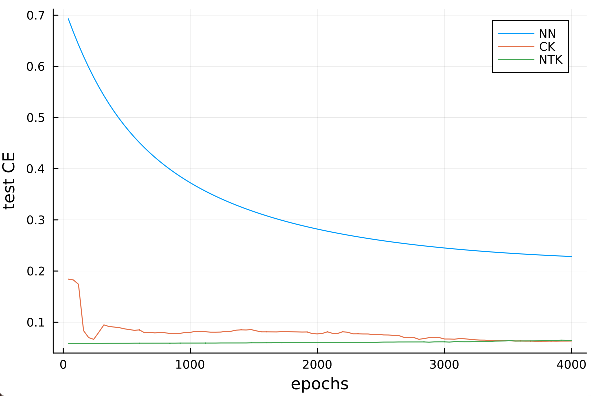}
\label{SFigKL1TanhEpochs}
}
{\includegraphics[width=\widb\textwidth]{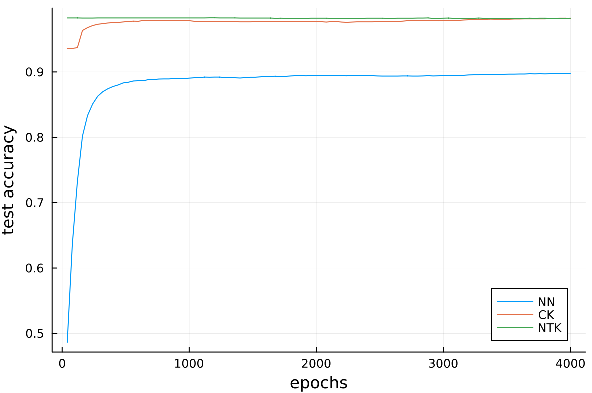}
\label{SFigKL1TanhEpochsAcc}
}}
\subfigure[$F_2(x_1,x_2) = 2x_1^3 - 0.6x_1^2 - 1.94x_1 + x_2  + 0.2 = 0$]
{{\includegraphics[width=\widb\textwidth]{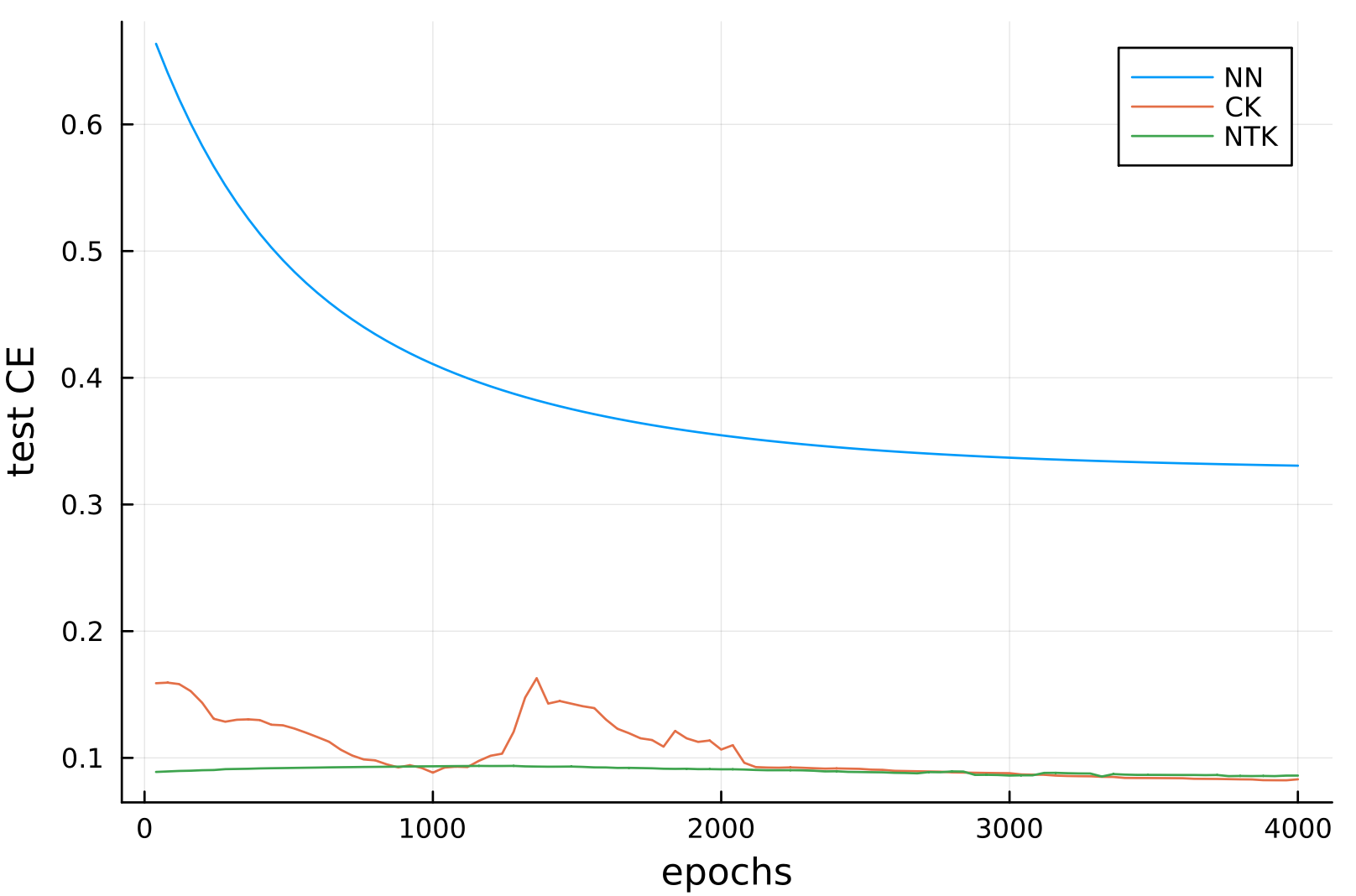}
\label{SFigKL2TanhEpochs}
}
{\includegraphics[width=\widb\textwidth]{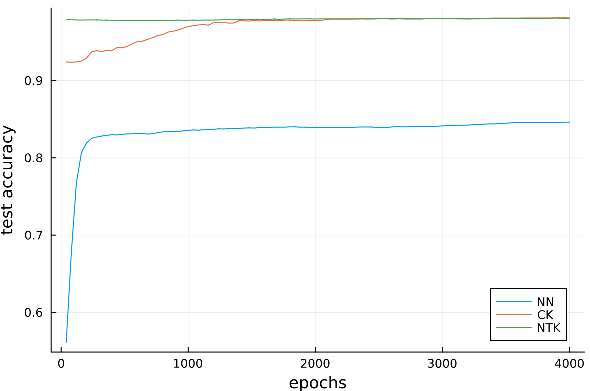}
\label{SFigKL2TanhEpochsAcc}
}}
\caption{Evolution of the test metrics, averaged over ten iterations, of NNs used to solve the logistic regression problem, and the corresponding NTK and CK results.}\label{FigKLTanhEpochs}
\end{figure}

We train NNs of the form \eqref{NN}, with $L = 2$ and $d_l = d = 128$ for $l = 1,2$, until the training accuracy (i.e., the proportion of correctly identified points) crosses $0.85$, or 4000 epochs have elapsed. We again employ the ADAM optimizer, with learning rate $10^{-5}$, and use Tanh as the primary activation function. The NTK and CK are extracted as detailed before and the respective training cross-entropies, defined in \eqref{KLTrLoss}, are minimized by using Newton's method to solve $\nabla_{\bs{\alpha}} \ell_{0,\text{ker}}(\bs{\alpha}) = \bs{0}$. We consider both cross-entropy loss values and classification accuracies as assessment metrics on the test points.

The plots in Figure \ref{FigKLTanh} show the results of 100 iterations of this procedure. A clear partition is visible in the performance metrics of the NN and the two kernels on the test datasets. At the same time, while the CK is generally outperformed by the NTK, the degree of improvement is dwarfed by the separation from the NN results. 

Figure \ref{FigKLTanh64} shows that this is still the case, even when the hidden widths of the NN are halved to 64, thus demonstrating that the CK performance is not simply a consequence of over-parameterization. However, Figure \ref{FigKLRelu} shows that the clear gap in performance vanishes when the activation function employed is ReLU. As in the case of function regression, this is explainable by over-fitting as a result of the volatile mix of over-parameterization and discontinuous feature map components. 

Finally, Figure \ref{FigKLTanhEpochs} shows the evolution of the performance metrics over the course of the NN training. We again show the average of ten iterations to reduce the effect of random initialization. The separation alluded to earlier is evident throughout, as is the relative proximity of the loss and accuracy values of the two kernels. Moreover, we note that the gaps shrink with training, so much so that the CK metrics are barely distinguishable from the NTK ones; the NN performance improves but a sizeable gulf nevertheless persists. In particular, this highlights a crucial point: even when the two kernels have been extracted from a minimally trained NN, they yield better results than the NN does at the end of training and when its metrics have plateaued. This underscores that using the CK (i.e., explicitly retraining only the last layer of the NN) is a low-cost recipe for significantly improving the accuracy of an NN.

\subsection{Physics-informed operator learning}
The viscous one-dimensional Burgers equation with periodic boundary conditions presented in \cite{wang2022improved, wang2021learning,howard2023multifidelity,howard2023stacked} presents a challenge for Deep Operator Networks (DeepONets) \cite{lu2021learning}. The system is given by
\begin{align}
  \textcolor{black}{  \frac{\partial s}{\partial t} + s \frac{\partial s}{\partial x} - \nu \frac{\partial ^2s}{\partial x^2}} &=    0, \; (x, t) \in [0, 1] \times [0, 1] \\
  s(x, 0) &= u(x), \; x\in[0, 1], \\
  s(0, t) &= s(1, t), \; t\in[0, 1], \\
  \frac{\partial s}{\partial x}(0, t) &= \frac{\partial s}{\partial x}(1, t), \; t\in[0, 1],
\end{align}
where $\nu$ is the viscosity. To provide a direct comparison, we run samples with hyperparameters as close as possible to \cite{wang2022improved}. We generate initial conditions $u(x)$ from a Gaussian random field $\sim \mathcal{N}(0, 25^2(-\Delta + 5^2I)^{-4})$, and the initial conditions are sampled at $P_{IC} = 101$ uniformly spaced locations on $x \in [0, 1]$, and the boundary conditions are randomly sampled at $P_{BC} = 100$ locations on $(x, t) = (0, t)$ and $(x, t) = (1, t)$. The residual is evaluated on $P_p = 2,500$ randomly sampled collocation points from the interior of the domain. The performance is tested by generating an additional $N_{test} = 500$ initial conditions $u(x)$ and simulating the solution with Matlab using the Chebfun \cite{Driscoll2014} package following the method in \cite{wang2022improved}. We train with $N=1000$ samples of the initial condition $u(x)$. 

In \cite{wang2022improved}, the use of the NTK to assign weights dynamically (for each point included in the loss function and for each training iteration) in front of the various terms in the loss function used to train the DeepONet was shown to perform better than using fixed predetermined weights. However, the evaluation of the NTK can be expensive. Alternatively, we can use the CK to assign the weights in the loss function.  Results comparing the CK, the NTK, and using predetermined fixed weights are shown in Table \ref{tab:Viscous_burgers_error}. The use of the CK not only leads to an improvement of the test accuracy compared to the NTK but also cuts the training time by around 62\% of the training time needed in the case of the NTK.

 \begin{table}[t!]
 \centering
 \caption{Comparison of the results and computational time for the viscous Burgers equation for $\nu = 0.001$. The reported error is the relative $\ell_2$ error over 500 randomly sampled initial conditions not seen in training. The results use local NTK and CK weights. as defined in \cite{wang2022improved} for the NTK. The run time is calculated using one node of an NVIDIA A100 GPU. }
 \begin{tabular}{lccc}
   Weights    & Fixed & NTK & CK \\
 \midrule
Error & $ 9.178 \% \pm 8.027 \% $ & $ 6.301 \% \pm 6.876 \% $ & $ 3.079 \% \pm 4.375 \% $ \\
Time (h) & 1.576 &  6.051 &  2.288 \\
 \bottomrule
 \end{tabular}
 \label{tab:Viscous_burgers_error}
 \end{table}

 \begin{figure}
 \centering
 \includegraphics[width=.8\linewidth]{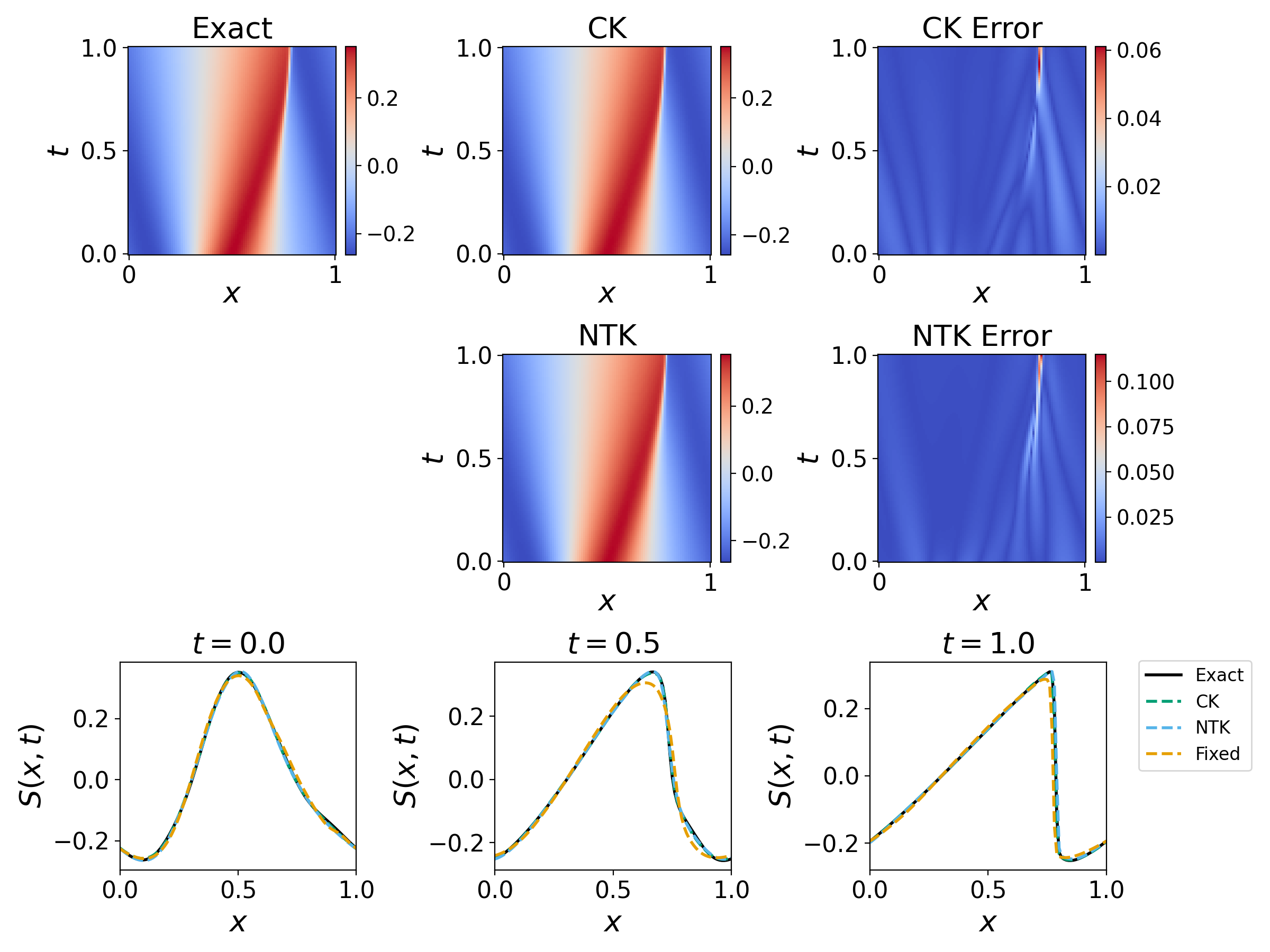}
 \caption{Results for viscous Burgers equation for $\nu = 0.001$.}
 \label{fig:Viscous_burgers_fig}
 \end{figure}

\subsection{Computer Vision Classification}

A small four layer CNN is trained to classify the first two classes of both the FMNIST and CIFAR  datasets. At regular intervals in training, the weights are stored. Offline, the stored weights are used to calculate the CK and NTK for the train and test datasets. Then, a kernel logistic regression model is trained on both the CK and NTK using the same train-test split as the neural network. We repeat this experiment 33 times with different seeds to evaluate the distribution of trained models, CKs, and NTKs. In figure~\ref{fig:LPvsNNCNN} we report the performance of the NN, NTK logistic regression model, and CK logistic regression model over the history of training. Similar to \cite{fort2020deep}, we find that the NTK logistic regression model quickly saturates to nearly the same final performance of the NN classifier. We also observe that the CK logistic regression model also becomes performant earlier in training than the NN model, but typically slower than the NTK model. However, given the computational burden of calculating the NTK compared to the CK (which we reiterate requires no-backward passes) we believe the CK could in fact be used to save the cost of training the NN. 

\begin{figure}
\centering
\includegraphics[width=.9\linewidth]{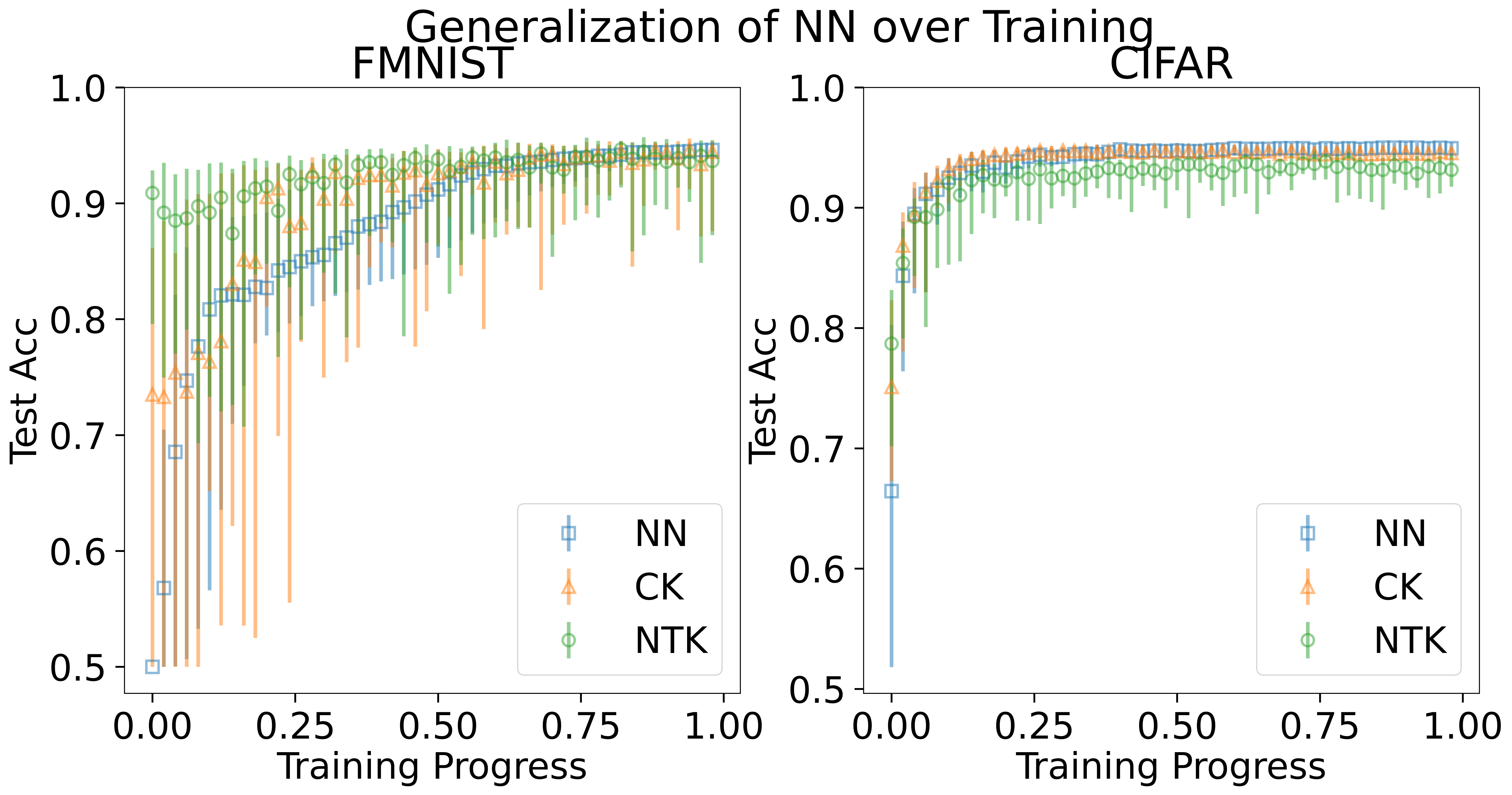}
\caption{A small CNN is trained to classify the first two classes of each the FMNIST and CIFAR datasets. After every epoch $t$, we freeze all weights except for the linear regression head and train the head for an additional 5 epochs (LP). We calculate the $\text{CK}_t$ and $\text{NTK}_t$ then train a kGLM then compare each model against the NN for generalization performance on the held out test data. We do this 33 times with different random seeds, uncertainty bars show the 1-std equivalent quantile. The figure shows across each model that after an initial feature learning phase, the NTK and CK kernel models become generalizable \textit{faster than} the neural network.}
\label{fig:LPvsNNCNN}
\end{figure}



\section{Application to Foundation Models}\label{SecFoundModels}

\begin{figure}
    \centering
\includegraphics[width=0.5\textwidth]{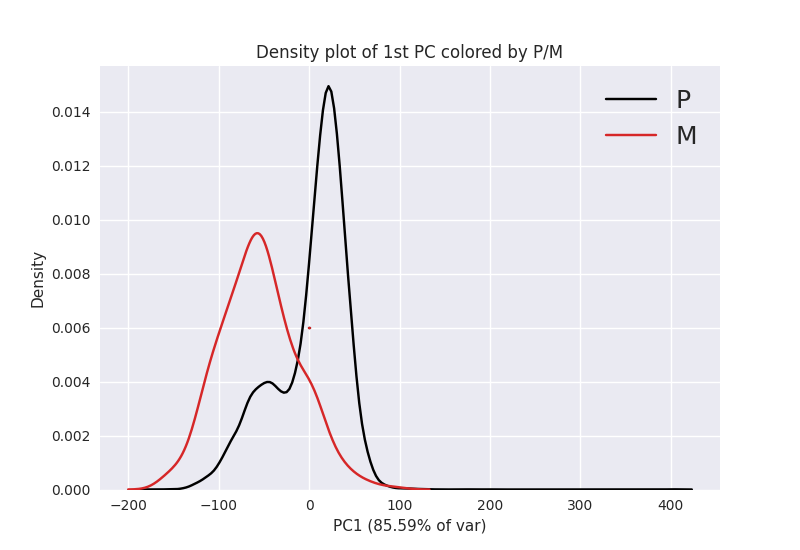}
    \caption{Kernel density estimate of the largest principal component of GPT-2's pre-trained final
embedding representation of Sentiment140 $P$ (the processed training set) and Sentiment140 $M$ (the manually curated test set).}
    \label{fig:pc1Dens}
\end{figure}

In this section, we demonstrate the effectiveness of our prescribed training strategy on a foundation model. Specifically, we make use of the pre-trained GPT-2 \cite{Radford2019LanguageMA}, Falcon 7B \cite{refinedweb}, and Llama-2 7B \cite{Touvron2023Llama2O} models for sentiment classification on a human-annotated subset of the Sentiment140 dataset consisting of 359 tweets, all of which are labelled either positive or negative based on the expressed sentiment\cite{Go2009TwitterSC}, rather than the entire Sentiment140 dataset due to a distribution shift between the train and test split as reported in \cite{TsouDistShift} and seen in Figure~\ref{fig:pc1Dens}. We evaluate two approaches: 
\begin{itemize}
\item[(FT)] We fine-tune by starting from pre-trained weights, attaching a randomly initialized classification head and train for ten epochs, and allowing \emph{all} the weights to update. 

\item[(LP)] We perform linear probing by extracting a feature vector at the final hidden layer and training a logistic regression classifier for ten epochs. 
\end{itemize}

We note that these procedures can be seen as analogues of the NN and CK training respectively studied in the earlier sections. In particular, both NN training and FT entail updates to all the parameters in the architecture, while the CK and LP approaches employ the feature map from the last hidden layer of a partially-trained architecture for logistic regression. Consequently, while the latter approach is significantly less expensive than the former, it can, as shown in the earlier sections, yield comparable (when LP substitutes for FT) or even a marked improvement in accuracy (when LP is combined with FT). 

We use pre-trained weights from Hugging Face for the transformer layers \cite{DBLP:journals/corr/abs-1910-03771}. Furthermore, to establish a baseline, we also perform linear probing with the LLMs initialized with random weights in the final classification head drawn from the truncated normal distribution $\mathcal{N}(0,0.02)$ \cite{Radford2018ImprovingLU},\cite{Radford2019LanguageMA}. All experiments are performed 30 times with the training and testing data shuffled. The averaged results, along with the standard errors, are reported in Table \ref{TabFoundMod} with further details about the hyperparameters provided in Appendix \ref{AppHyperpara}, Table \ref{tab:Hyperparameters}. We note that exclusively training the classification head (LP) on the in situ feature representation of the pre-trained model attains at least $96\%$ of the generalization performance of updating all of the weights (FT) for all three LLMs. QLoRA outperforms FP for all LLMs with nearly an order of magnitude increase in computational cost with a more modest memory saving.  

\begin{table}[t!]
    \centering
    \caption{Performance metrics of various training methods for Falcon 7B, Llama-2 7B and GPT-2 on classification of Sentiment140's manually curated data. Note that LP in this table corresponds to only training the logistic regression layer without any fine-tuning of the model. }
    \begin{tabular}{lccccc}
         Model & Method  & Train Acc. & Test Acc. & Time/Epoch (s) & GPU Memory (GB)\\
         \midrule
         Falcon 7B & FT & $99.95\pm0.02\%$& $91.8\pm0.9\%$& 103.9& 56.5\\
         & LP & $96.4\pm0.12\%$& $88.9\pm0.8\%$& 8.8& 14.3\\
         & QLoRA & $99.62\pm0.07\%$& $93.6\pm0.6\%$& 128.08& 7.4\\
         \midrule
         Llama-2 7B & FT & $99.81\pm0.05\%$& $90.5\pm 0.75\%$& 105.27& 53.5\\
         & LP & $97.17\pm0.12\%$& $ 87.4\pm 0.57\%$& 9.49& 13.8\\
         & QLoRA & $96.33\pm0.33\%$& $87.68\pm 0.79\%$& 135.5& 6.5\\
         \midrule
         GPT-2 & FT & $99.84\pm0.03\%$& $86.67\%\pm0.94\%$& 12.72& 3.7\\
         & LP & $93.18\pm0.24\%$& $85.91\pm0.94\%$ & 3.66& 1.87\\
         & QLoRA & $99.20\pm0.11\%$& $ 88.81\%\pm 0.65\%$& 29.03& 1.61\\
         \bottomrule
    \end{tabular}
    \label{TabFoundMod}
\end{table}

\begin{figure}
    \centering
    \includegraphics[width=0.3\linewidth]{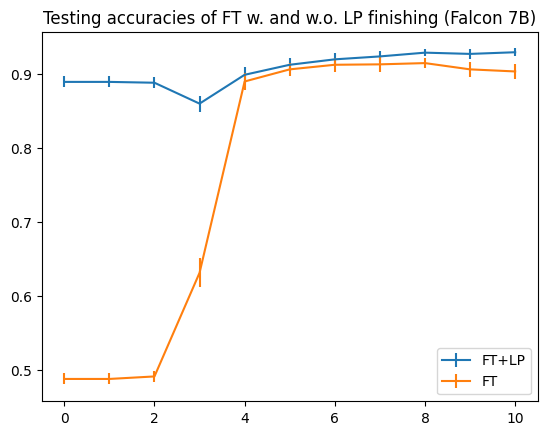}
    \includegraphics[width=0.3\linewidth]{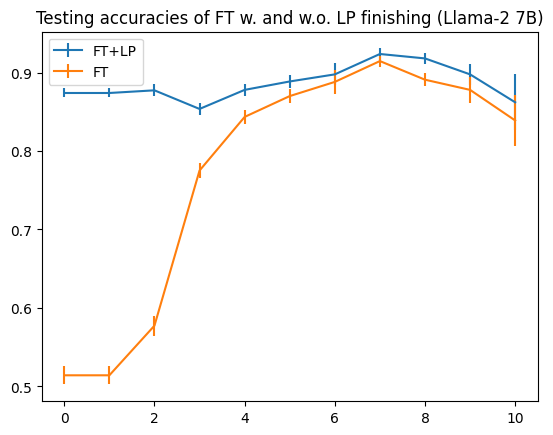}
    \includegraphics[width=0.3\linewidth]{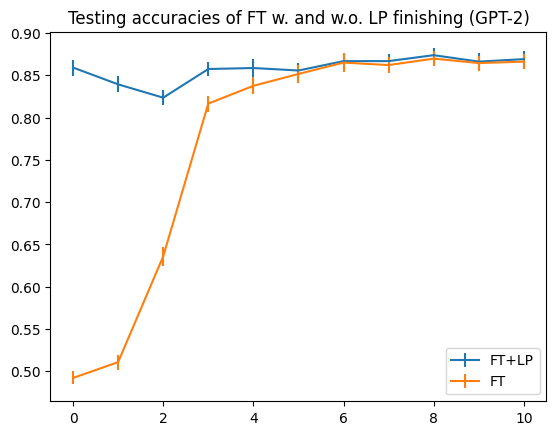}
    \includegraphics[width=0.3\linewidth]{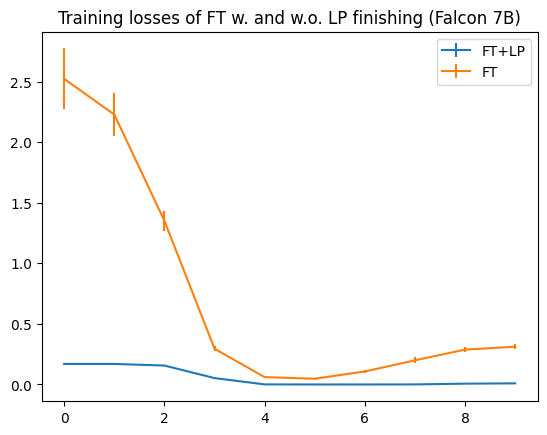}
    \includegraphics[width=0.3\linewidth]{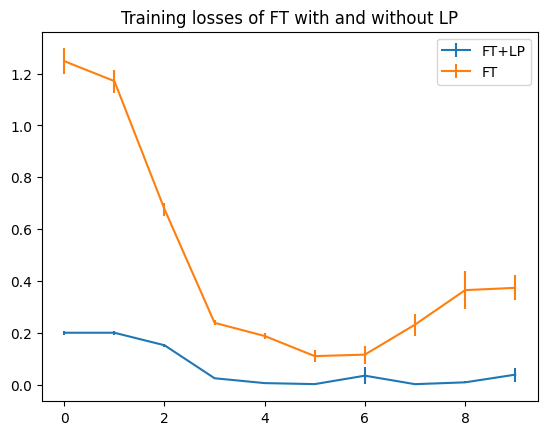}
    \includegraphics[width=0.3\linewidth]{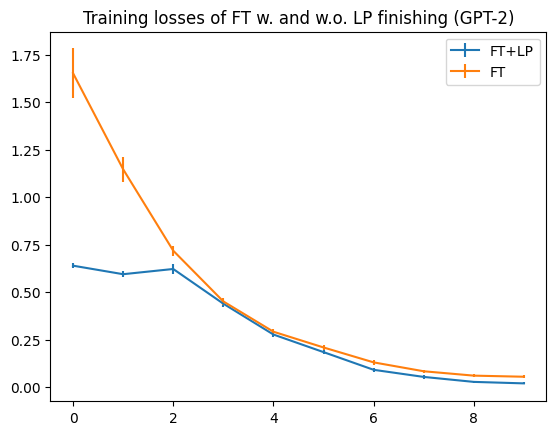}
    \caption{Test accuracy of training LLMs over 10 epochs. After each epoch, we freeze all weights except for the final layer and train the final layer for 10 additional epochs. We plot averages over 30 seeds, error bars indicating standard error of the mean. Note that LP in Table~\ref{TabFoundMod} 
    corresponds to the left most point (Epoch 0) while FT in Table~\ref{TabFoundMod} corresponds to the right most point (Epoch 10) for the testing accuracies.}  
    \label{fig:NN-finishedoff}
\end{figure}
Due to the short training time and low memory cost of training an LP, we investigate the effect of quickly training a linear probe after any given amount of FT training epochs (including the case of no FT given at epoch 0). Figure \ref{fig:NN-finishedoff} shows the average test accuracies of FT training after each epoch, compared to the performance of the model if we had stopped FT training and instead trained 10 epochs with LP. In all cases, we observe that on average, topping off the FT procedure with some additional LP epochs envelopes accuracy from the top as well as loss from below. These bounds on the test and training accuracies coupled with gains in wall clock performance as well as memory allocation (compared with FT) makes a strong argument for limiting the number of FT epochs in exchange for a top off of LP training. 


\section{Discussion}\label{SecDisc}
Despite possessing the universal approximation property in theory \cite{cybenko1989approximation,hornik1989multilayer} and enjoying the expressiveness that accompanies increased depth \cite{daubechies2022nonlinear}, in practice neural networks are stymied by optimization errors in their efforts to approximate given functions to arbitrary accuracy. Finding the optimal parameters requires the navigation of highly complex and non-convex loss landscapes and demands a large number of training epochs, thus driving up the training cost. Furthermore, the somewhat cumbersome architectures do not admit closed-form characterizations of the learned functions, making it harder to determine their generalization properties, e.g., assess them on unseen data or predict their stability under the addition of seemingly negligible noise.

Since the NTK controls the evolution during training of the learned function, it provides one way of supplying the missing characterization. In particular, an optimally trained NN, in the infinite width limit, can be shown to equal a kernel machine with a closed form known {\it a priori}. From \eqref{NN_NTK}, it follows that, for finite width networks, a kernel machine relying on the empirical NTK would necessarily outperform the NN on the training data. The numerical tests in Section \ref{SecNumTests} confirm that this is also the case on test datasets across different problems, network widths, and training stages, provided the Tanh activation function is used. However, evaluating the resulting approximations (e.g., \eqref{KRApprox} and \eqref{KLPred}) at a new point requires computing the NTK by pairing this point with every training point, which is a highly expensive undertaking for networks used in practical applications. Worse still, the NTK performance may be inferior to that of the NN when ReLU activations are employed (e.g., Figures \ref{FigKRRelu} and \ref{FigKLRelu}) due to discontinuities in its feature map components. Furthermore, the resulting sensitivity to small perturbations in the inputs may also leave these approximations singularly susceptible to adversarial attacks.

In this paper, we studied the properties of approximations relying on the CK and found, both theoretically and empirically, that their performance is closely tied to the NTK results, provided certain regularity conditions on the approximations are met. These are satisfied by default for the Tanh activation function, thus establishing that the gains in accuracy for the NTK over the NN are also inherited by the CK approximations. Importantly, the training cost of a CK machine is significantly lower than for continuing to train the NN, yielding a closed form formula for function regression \eqref{KRApprox} and a convex optimization problem for logistic regression \eqref{KLTrLoss}. In addition, its evaluation on new data-points is as inexpensive as evaluating the NN on the same points. As a result, the CK achieves a boost in accuracy over the NN comparable to that of the NTK while being strikingly cheaper to use than the latter and much easier to train than the former.

The CK also possesses features that, in some settings, make it even more favourable than the NTK. Due to the relatively low dimension of its feature space, the usage of its feature map form \eqref{KREqApprox} for function regression is computationally viable. This approximation is much better conditioned than the corresponding kernel form, with the result that round-off errors are not magnified excessively and the approximation errors can achieve a much lower plateau, as shown in Figures \ref{FigKRTanh}, \ref{FigKRTanhEpochs} and \ref{FigKRTanh_256}. In addition, the components of the feature map corresponding to the CK are continuous, even when using ReLU activations. As a consequence, the resulting approximations are less prone to overfitting and more robust to noise in the inputs.

We note that training a CK machine is equivalent to finding the optimal weights to attach to the neurons in the last hidden layer.  Our work then suggests that a low-cost strategy for improving the accuracy of a trained NN is to simply retrain the last layer of parameters. This recipe does not alter the form of the neurons (i.e., the feature map components for the CK) but, as we have shown, can lead to significant enhancements in approximation accuracy. The regularity of the activation function is a key determinant in the success of this strategy. Observe that the choice of the activation function does not appear to affect the NN approximation errors too much, but the kernel performance is closely tied to it (e.g., compare Figures \ref{FigKRTanh} and \ref{FigKRRelu}, and Figures \ref{FigKLTanh} and \ref{FigKLRelu}). This indicates that the effects of the regularity of the activation function are realized only on the kernel machines, and that the NN is largely agnostic to it. 

In addition to simple illustrative examples, we have tested the performance of CK compared to the NTK also for more elaborate cases. First, we used the CK approximation to replace the NTK as a way of assigning weights for  physics-informed operator learning (a regression task). Our numerical results show that the CK can lead to comparable performance to the NTK while reducing significantly the training time. Second, we used the CK and NTK to train surrogates for logistic regression for vision classification (FMNIST and CIFAR). We show that the NTK and CK both become performant earlier in training compared to the NN classifier. Even though the CK becomes performant more slowly than the NTK, the computational efficiency of CK compared to NTK makes a case for its use.

We also deployed our strategy on a foundation model for classification illustrates that freezing the feature map and training only the last layer of parameters yields performance that compares favourably with full architecture training. In addition, this approach might even lend improved robustness to the predictive capabilities in that the limited degrees of freedom resist overfitting on the training data. Our results therefore hint at the possibility that early stopping for model fitting might be beneficial in more ways than is currently surmised and that there might be subtle procedures to modulate stopping times across various parts of the network. Our recommended approach offers a computational speed up without the large trade-off in accuracy as seen in other efficient fine-tuning methods.

Finally, our results can also offer an explanation as to why feed-forward regression networks perform so well. In essence, the information captured in the last hidden layer's feature representation uniquely determines the CK. While lossy, the CK still preserves enough information from the global feature representation encoded in the NTK. And so by construction, these neuro-architectures learn and then distill the relevant information by the final hidden layer for the regression tasks. Our use of approximation theory not only offers an elegant reason as to why these networks work but also a clever way to enhance them as well.

\section*{Acknowledgements}
The work of SQ is supported by the Department of Energy (DOE) Office of Advanced Scientific Computing Research (ASCR) through the Pacific Northwest National Laboratory Distinguished Computational Mathematics Fellowship (Project No. 71268). The work of AD, MV, AE, and TC were partially supported by the Mathematics for Artificial Reasoning in Science (MARS) initiative via the Laboratory Directed Research and Development (LDRD) Program at PNNL. The work of AH and PS is partially supported by the U.S. Department of Energy, Office of Science, Advanced Scientific Computing Research program
under the Scalable, Efficient and Accelerated Causal Reasoning Operators, Graphs and Spikes for Earth and Embedded Systems (SEA-CROGS) project (Project No. 80278). The computational work was partially performed using PNNL Institutional Computing.
Pacific Northwest National Laboratory is a multi-program national
laboratory operated for the U.S. Department of Energy by Battelle Memorial Institute
under Contract No. DE-AC05-76RL01830.

\bibliography{NTKrefs}

\appendix

\section{Computing the NTK}\label{AppNTKComp}

Consider a fully-connected network (NN) with $L$ hidden layers and layer widths $(d_0, d_1, \hdots, d_{L}, d_{L+1})$ of the form given in \eqref{NN}. We rewrite it as
\eqn{
y^{(l)} &=& W^{(l)}x^{(l-1)} + b^{(l)}, \quad 1 \leq l \leq L+1, \nonumber\\
x^{(l)} &=& \sigma\left(y^{(l)}\right), \qquad 1 \leq l \leq L, \label{NNApp}
}
to introduce the pre-activation values $\{y^{(l)}\}$. As in Section \ref{SecPrelim}, $\sigma$ is the activation function, $W^{(l)} \in \mathbb{R}^{d_l \times d_{l-1}}$ and $b^{(l)} \in \mathbb{R}^{d_l}$ are the parameters, $x^{(0)} \in \mathbb{R}^{d_0}$ is the input, and $x^{(L+1)} = y^{(L+1)} \in \mathbb{R}^{d_{L+1}}$ is the output. Note that we do not apply the activation function in the last layer. Thus, this network represents a map $f_\text{NN}: \mathbb{R}^{d_0} \to \mathbb{R}^{d_{L+1}}$ by $f_\text{NN}(x^{(0)}) = x^{(L+1)}$.

Given two inputs $x,\hat{x} \in \mathbb{R}^{d_0}$, we define the NTK matrix for this network as
\eqn{
K(x,\hat{x}) = \left(K_{ij}(x,\hat{x})\right)_{1 \leq i,j \leq d_{L+1}}, \label{NTKmat}
}
where
\eqn{
K_{ij}(x,\hat{x}) = \sum_{l = 1}^{L+1} \left\{ \sum_{k = 1}^{d_l} \frac{\partial y^{({L+1})}_i}{\partial b^{(l)}_k} \frac{\partial \hat{y}^{({L+1})}_j}{\partial b^{(l)}_k}  + \sum_{k = 1}^{d_l} \sum_{m = 1}^{d_{l-1}} \frac{\partial y^{({L+1})}_i}{\partial W^{(l)}_{km}} \frac{\partial \hat{y}^{({L+1})}_j}{\partial W^{(l)}_{km}}  \right\}, \label{NTKdef}
}
with $y^{({L+1})} = f_\text{NN}(x)$ and $\hat{y}^{({L+1})} = f_\text{NN}(\hat{x})$.

For brevity, we denote by $\otimes$ the ``column-wise product'' between a matrix $A_{m \times n}$ and vector $b_{n \times 1}$ given by
\eqn{
A \otimes b := A\left(\text{diag}(b)\right); \label{Colwise}
}
the result is then an $m \times n$ matrix.

\subsection{Applying the chain rule}\label{SubSecChainRules}
Let $1 \leq r \leq {L+1}$. From \eqref{NN}, it follows that
\eqn{
\frac{\partial y_i^{(r)}}{\partial b^{(r)}_k} &=& \delta_{ik}, \quad 1 \leq i,k \leq d_{r}, \nonumber\\
\frac{\partial y_i^{(r)}}{\partial W^{(r)}_{km}} &=& \delta_{ik}x^{(r-1)}_m, \quad 1 \leq i,k \leq d_{r}, \ 1 \leq m \leq d_{r-1}. \label{PDrr}
}

Next note that, for $1 \leq r < s \leq {L+1}$, we have
\eqn{
\frac{\partial y_i^{(s)}}{\partial b^{(r)}_k} = \sum_{n = 1}^{d_{s-1}} \frac{\partial y_i^{(s)}}{\partial x^{(s-1)}_n} \frac{\partial x_n^{(s-1)}}{\partial b^{(r)}_k}, \qquad 1 \leq i \leq d_s. \label{PDbsr1}
}

Using $\frac{\partial y_i^{(s)}}{\partial x^{(s-1)}_n} = W^{(s)}_{in}$ and
\eqn{
\frac{\partial x_n^{(s-1)}}{\partial b^{(r)}_k}  = \frac{\partial x_n^{(s-1)}}{\partial y^{(s-1)}_n} \frac{\partial y_n^{(s-1)}}{\partial b^{(r)}_k} = \sigma'(y_n^{(s-1)}) \frac{\partial y_n^{(s-1)}}{\partial b^{(r)}_k} \nonumber 
}
in \eqref{PDbsr1}, we obtain
\eqn{
\frac{\partial y_i^{(s)}}{\partial b^{(r)}_k} = \sum_{n = 1}^{d_{s-1}} W^{(s)}_{in} \sigma'(y_n^{(s-1)}) \left[\frac{\partial y_n^{(s-1)}}{\partial b^{(r)}_k} \right]. \label{PDbsr2}
}

Similarly, we arrive at, for $1 \leq i \leq d_s$,
\eqn{
\frac{\partial y_i^{(s)}}{\partial W^{(r)}_{km}} = \sum_{n = 1}^{d_{s-1}} \frac{\partial y_i^{(s)}}{\partial x^{(s-1)}_n} \frac{\partial x_n^{(s-1)}}{\partial W^{(r)}_{km}} = \sum_{n = 1}^{d_{s-1}} W^{(s)}_{in} \sigma'(y_n^{(s-1)}) \left[\frac{\partial y_n^{(s-1)}}{\partial W^{(r)}_{km}} \right]. \label{PDWsr2}
}

Setting 
\eqn{
\frac{\partial y^{(s)}}{\partial b^{(r)}_{k}}  = \left(\frac{\partial y_i^{(s)}}{\partial b^{(r)}_{k}} \right)_{1 \leq i \leq d_s} , \quad \frac{\partial y^{(s)}}{\partial W^{(r)}_{km}}  = \left(\frac{\partial y_i^{(s)}}{\partial W^{(r)}_{km}} \right)_{1 \leq i \leq d_s} \nonumber
}
allows us to rewrite \eqref{PDbsr2} and \eqref{PDWsr2} as
\eqn{
\frac{\partial y^{(s)}}{\partial b^{(r)}_{k}} = [W^{(s)} \otimes \sigma'(y^{(s-1)})] \left[\frac{\partial y^{(s-1)}}{\partial b^{(r)}_k} \right] , \quad
\frac{\partial y^{(s)}}{\partial W^{(r)}_{km}} = [W^{(s)} \otimes \sigma'(y^{(s-1)})] \left[\frac{\partial y^{(s-1)}}{\partial W^{(r)}_{km}} \right]. \label{PDsr}
}

Applying these expressions iteratively and using \eqref{PDrr} yields
\eqn{
\frac{\partial y_i^{(s)}}{\partial b^{(r)}_{k}} = \left\{ [W^{(s)} \otimes \sigma'(y^{(s-1)})] [W^{(s-1)} \otimes \sigma'(y^{(s-2)})] \hdots [W^{(r+1)} \otimes \sigma'(y^{(r)})]\right\}_{ik} \label{PDbsr3}
}
and
\eqn{
\frac{\partial y_i^{(s)}}{\partial W^{(r)}_{km}} = \left\{ [W^{(s)} \otimes \sigma'(y^{(s-1)})] [W^{(s-1)} \otimes \sigma'(y^{(s-2)})] \hdots [W^{(r+1)} \otimes \sigma'(y^{(r)})]\right\}_{ik}x_m^{(r-1)}. \label{PDWsr3}
}

This implies that the contribution to the NTK from parameters $\{W^{(r)},b^{(r)}\}$ is 
\eqn{
(1 + x^{(r-1)}\cdot \hat{x}^{(r-1)}) \left\{ \Pi_{s = {L+1}}^{r+1} [W^{(s)} \otimes \sigma'(y^{(s-1)})]\right\} \left\{ \Pi_{s = {L+1}}^{r+1} [W^{(s)} \otimes \sigma'(\hat{y}^{(s-1)})]\right\}^\top, \label{NTKcont}
}
where the product indices are decreasing.

\subsection{Two algorithms}\label{SubSecAlgos}

The structure evident in \eqref{NTKcont} allows us to devise two efficient algorithms for computing the NTK. The first scheme proceeds in the more obvious way by iteratively assembling the matrices
\eqn{
\Pi_{s = {L+1}}^{r+1} [W^{(s)} \otimes \sigma'(y^{(s-1)})] , \quad  \Pi_{s = {L+1}}^{r+1} [W^{(s)} \otimes \sigma'(\hat{y}^{(s-1)})] \nonumber 
} 
moving backwards and computing their contributions to the NTK. This approach requires the storage of some intermediate terms. In contrast, the second algorithm we present proceeds in tandem with a forward pass and therefore does not require the storage of intermediate results.

\begin{algm}\label{algmB}
\begin{algorithmic}
\\
\State 1. Initialize $Y = \hat{Y} = I_{d_{L+1}}$ and set 
\eqn{
K = \left(1 + x^{(L)}\cdot \hat{x}^{(L)}\right)I_{d_{L+1}}. \nonumber
}
\State 2. For $l = {L+1}, L, \hdots, 2$, update
\eqn{
Y \to Y\left[W^{(l)}\otimes \sigma'(y^{(l-1)})\right], \quad \hat{Y} \to \hat{Y}\left[W^{(l)}\otimes \sigma'(\hat{y}^{(l-1)})\right] \label{Yupdates}
}
and update
\eqn{
K \to K + \left(1 + x^{(l-2)}\cdot \hat{x}^{(l-2)}\right)Y\hat{Y}^\top. \nonumber
} 
\end{algorithmic}
\end{algm}

Observe that this routine requires the storage of inner products $\{x^{(l)}\cdot \hat{x}^{(l)}\}_{0 \leq l \leq L}$ and the derivative of the activation function evaluated at the pre-activation values $\{y^{(l)}\}_{1 \leq l \leq L}$. The former only scales with the number of layers while the latter is the same as the number of hidden neurons. However, in the case that the activation function is ReLU, we can achieve considerable savings by noting that its derivative is the Heaviside function and storing the derivatives as logical variables. This can also be used to speed up the computations in \eqref{Yupdates} by only employing the matrix entries that correspond to positive pre-activation values. 

For a specific case, following \cite{novak2022fast}, assume that $d_0 = \hdots = d_{L} = W$ and $d_{L+1} = O$, and we need to compute the NTK for a batchsize of $N$ inputs. Then, Algorithm \ref{algmB} requires the storage of 
\begin{itemize}
\item the intermediate matrices $Y$ and $\hat{Y}$ at cost $NOW$,
\item the intermediate NTK matrix $K$ at cost $N^2O^2$,	
\item the layer-wise inner products $\{x^{(l)}\cdot \hat{x}^{(l)}\}$ at cost $N^2L$ (or at $NLW$ by storing the $\{x^{(l)}\}$ values),
\item the pre-activations evaluated at the derivative of the activation function at $NLW$. 
\end{itemize}

Thus, the total memory requirements are on the order of $NOW + N^2O^2 + N^2L + NLW$. In addition, the computational costs involve
\begin{itemize}
\item updates to the intermediate matrices $Y$ and $\hat{Y}$ at cost $NOLW^2$,
\item updates to the intermediate NTK matrix $K$ at cost $N^2O^2LW$,	
\end{itemize}
for a total of $NOLW^2 + N^2O^2LW$.

The second algorithm hinges on the observation that summing up terms of the form \eqref{NTKcont} admits a nested representation in the spirit of Horner's rule for polynomial evaluation. In particular, this is facilitated by the successive matrices $[W^{(l)} \otimes \sigma'(y^{(l-1)})]$ and $[W^{(l)} \otimes \sigma'(\hat{y}^{(l-1)})]^\top$ being appended on the outside.  

\begin{algm}\label{algmF}
\begin{algorithmic}
\\
\State 1. Initialize $K = (1 + x^{(0)} \cdot \hat{x}^{(0)})I_{d_1}$. 
\State 2. For $l = 2,\hdots,{L+1}$, update
\eqn{
K \to \left[W^{(l)} \otimes \sigma'(y^{(l-1)})\right]K\left[W^{(l)} \otimes \sigma'(\hat{y}^{(l-1)})\right]^\top + \left(1 + x^{(l-1)} \cdot \hat{x}^{(l-1)} \right)I_{d_{l}}. \nonumber 
}
\end{algorithmic}
\end{algm}

Note that this scheme only makes use of the pre- and post-activation values at each step while moving forward, requiring no entries to be stored from previous layers. Using the example considered earlier, this scheme requires the storage of the intermediate matrices $K$ at cost $N^2W^2$. The computational expense meanwhile is on the order of $N^2LW^3 + N^2OW^2 = N^2W^2(LW + O)$ for the matrix multiplications.

To sum up, we recommend the use of Algorithm \ref{algmB} if $O < W$ and Algorithm \ref{algmF} if $O > W$.

\section{Identifying the orthonormal basis}\label{AppBFs}
The kernel approximation \eqref{KRApprox} can be interpreted as an orthogonal projection on the space 
\eqn{
\mathcal{S}_\text{ker} = \text{span}\left(\left\{\Phi_{\text{ker},j} \right\}_{j = 1}^K\right) \nonumber
}
with respect to the inner product
\eqn{
\ip{h_1,h_2}_0 = \sum_{i = 0}^{N} h_1(\bs{x}_i)h_2(\bs{x}_i)w_i. \label{IPdefn}
} 

More precisely, given a target function $f:Z \to \mathbb{R}$, we have
\eqn{
\ip{\Phi_{\text{ker},l},f - f_\text{ker}}_0 = 0, \qquad 1 \leq l \leq K. \label{KerProj}
}

The optimal representation of this projection requires an orthonormal basis, which the given functions $\left\{\Phi_{\text{ker},j} \right\}_{1 \leq j \leq K}$ may not necessarily form with respect to the inner product. We note however that computing the SVD 
\eqn{
W^{1/2} \Xi_\text{ker}^\top = USV^\top \label{SVDFeature}
}
yields the values of a hierarchical orthonormal basis at the training nodes, as well as a change of basis formula: defining
\eqn{
u_j(\bs{z}) = s_j^{-1}\sum_{k = 1}^{K} {V}_{kj}\Phi_{\text{ker},k}(\bs{z}), \quad 1 \leq j \leq r, \label{COBform}
}
where $r$ is the rank of $\Xi_\text{ker}$, leads to $u_j(\bs{x_i}) = (W^{-1/2} {U})_{ij}$, with the result that
\eqn{
\ip{u_{j},u_l}_0 = \sum_{i = 0}^{N} u_j(\bs{x_i}) u_l(\bs{x_i})w_i = \left(U^\top U\right)_{jl} = \delta_{jl}. \label{OrthBasis}
}

Furthermore, we have
\eqn{
s_j^2 = \sum_{k = 1}^{K} \left|\ip{u_j,\Phi_{\text{ker},k}}_0 \right|^2, \label{SingVals}
}
so the monotonicity of the singular values $s_1 \geq s_2 \geq ...$ is indicative of the hierarchy in the orthonormal basis.

Computing the orthonormal basis functions does not necessarily require the Jacobian explicitly -- instead, they can also be inferred directly from the SVD $W^{1/2} H_\text{ker} W^{1/2} = U S^2 U^\top$: plugging $V = \Xi_\text{ker}W^{1/2} US^{-1}$ into \eqref{COBform} yields the formula
\eqn{
u_j(\bs{z}) = s_j^{-2} \ip{u_j , \text{ker}(\cdot , \bs{z})}_0, \label{COBform2}
}
although the division by the squares of the small singular values may corrupt the corresponding basis functions. The optimal representation orthogonal projection with respect to these basis functions is then given by
\eqn{
f_\text{ker}(\bs{z}) = \sum_{j = 1}^r \ip{f,u_j}_0 u_j(\bs{z}). \label{BFOrthogProj}
}

\section{Hyperparameters for sentiment analysis}\label{AppHyperpara}

Hyperparameters used for the linear probing and fine-tuning experiments for sentiment classification are shown in Table \ref{tab:Hyperparameters}. The loss function {\tt BCEWithLogitsLoss} combines a binary cross entropy loss and a sigmoid scaling operation. For QLoRA techniques, we used the {\tt bitsandbytes} package for quantization.

\begin{table}[!ht]
\centering
\begin{tabular}{lllll}
\toprule
Model & Hyperparameter & Linear Probing & Fine-Tuning & QLoRA \\
\midrule
All &Epochs & $10$ & $10$ &10\\
&Optimizer  & AdamW & AdamW & AdamW \\
&Batch Size & $1$ & $1$ & $1$\\
&Dropout & $0.1$ & $0.1$ & $0.1$ \\ 
&Gradient Norm Clip & $1.0$ & $1.0$ & 1.0 \\
&Seeds & $0-29$ & $0-29$& 0-29\\
&Loss Function & BCE+Logits & BCE+Logits & BCE+Logits\\ 
&Random Feature Init.  & \makecell[l]{Sample from $\mathcal{N}(0,0.02)$,\\ residual layers scaled by $\frac{1}{\sqrt{n}}$} & -- \\
&LoRA adaptors & -- & -- & All linear layers\\
&Quantization & -- & -- & \makecell[l]{NF4\\ Double quant.}\\
\midrule
GPT-2&Learning Rate & $10^{-2}$ & $10^{-5}$ & $10^{-4}$\\
&$\beta_1$ & $.9$ & $.9 $& $0.9$\\  
&$\beta_2$ & $.999$ & $.999$& $0.999$\\ 
&$\epsilon$ & $10^{-8}$ & $10^{-8}$ &$10^{-8}$\\ 
&Scheduler-Type & LinearLR & LinearLR& LinearLR\\ 
&Start Factor & $.33$ & $.33$ &0.33\\
&End Factor & $1.0$ & $1.0$ &1.0\\
&Scheduler Iterations & $5$ & $5$& 5\\
\midrule
Llama-2&Learning Rate & $10^{-2}$ & $5\times10^{-4}$& $10^{-4}$\\
&$\beta_1$ & $.9$ & $.9 $& $0.9$\\  
&$\beta_2$ & $.95$ & $.95$& $0.95$\\ 
&$\epsilon$ & $10^{-8}$ & $10^{-8}$& $10^{-8}$\\ 
&Scheduler-Type & Cosine & Cosine& Cosine\\ 
&Warm Up Steps & $150$ & $150$& $150$\\
\midrule
Falcon&Learning Rate & $10^{-2}$ & $5\times10^{-4}$ & $10^{-3}$\\
&$\beta_1$ & $.9$ & $.9 $ & $0.9$\\  
&$\beta_2$ & $.95$ & $.95$ & $0.95$\\ 
&$\epsilon$ & $10^{-8}$ & $10^{-8}$ & $10^{-8}$ \\ 
&Scheduler-Type & Cosine & Cosine & Cosine \\ 
&Warm Up Steps & $150$ & $150$ & $150$\\
\bottomrule
\end{tabular}
\caption{Hyperparameters used in the linear probing and fine-tuning approaches for sentiment classification.} \label{tab:Hyperparameters}
\end{table}


\end{document}